\documentclass[10pt,twocolumn,letterpaper]{article}
\usepackage[pagenumbers]{wacv}

\usepackage{microtype}
\usepackage{graphicx}
\usepackage{booktabs} 

\usepackage[pagebackref,breaklinks,colorlinks,bookmarks=false]{hyperref}

\usepackage{amsmath}
\usepackage{amssymb}
\usepackage{mathtools}
\usepackage{amsthm}

\usepackage[capitalize,noabbrev]{cleveref}

\theoremstyle{plain}
\newtheorem{theorem}{Theorem}[section]

\theoremstyle{definition}

\theoremstyle{remark}
\newtheorem{remark}[theorem]{Remark}

\usepackage[textsize=tiny]{todonotes}

\usepackage{tikz}
\usetikzlibrary{matrix,positioning,calc,shadings}
\usepackage{booktabs}
\usepackage{caption}
\usepackage{enumitem}  
\usepackage{bm}
\usepackage{multirow}

\graphicspath{ {figs/} }

\setlist[itemize]{leftmargin=*}
\setlist[enumerate]{leftmargin=*}

\newcommand{\DALLE}{{DALL\(\cdot\)E}}
\newcommand{\CelebAHQ}{{CelebA-HQ}}
\newcommand{\CelebA}{{CelebA}}

\newcommand{\RR}{\mathbb{R}}

\newcommand{\vi}{\bm{i}}
\newcommand{\vj}{\bm{j}}

\newcommand{\sD}{\mathcal{D}}

\newcommand{\set}[1]{\left\{#1\right\}}

\def\wacvPaperID{1291} 
\def\confName{WACV}
\def\confYear{2025}

\begin{document}

\title{GeoPos: A Minimal Positional Encoding for Enhanced Fine-Grained Details in Image Synthesis Using Convolutional Neural Networks}

\author{Mehran Hosseini\textsuperscript{*}\\
King's College London\\
London, United Kingdom\\
{\tt\small mehran.hosseini@kcl.ac.uk}
\and
Peyman Hosseini\thanks{Equal contribution; ordered alphabetically}\\
Queen Mary University of London\\
London, United Kingdom\\
{\tt\small s.hosseini@qmul.ac.uk}
}

\maketitle


\begin{abstract}
  The enduring inability of image generative models to recreate
  intricate geometric features, such as those present in human hands
  and fingers has been an ongoing problem in image generation for
  nearly a decade. While strides have been made by increasing model
  sizes and diversifying training datasets, this issue remains
  prevalent across all models, from denoising diffusion models to
  Generative Adversarial Networks (GAN), pointing to a fundamental
  shortcoming in the underlying architectures. In this paper, we
  demonstrate how this problem can be mitigated by augmenting
  convolution layers geometric capabilities through providing them
  with a single input channel incorporating the relative
  $n$-dimensional Cartesian coordinate system. We show this
  drastically improves quality of images generated by Diffusion Models,
  GANs, and Variational AutoEncoders (VAE).
\end{abstract}

\section{Introduction}
\label{sec: Introduction}
Generative models have gained immense popularity and generated
unprecedented hype in the last few years, revolutionising the way we
approach tasks that involve generating new content. SoA
image generative models, like \DALLE{} 3 \cite{Ramesh+21,
  Ramesh+22, Betker+23}, Stable Diffusion \cite{Rombach+22},
Midjourney \cite{Midjourney}, and Nvidia's StyleGAN \cite{KarrasLA19,
  Karras+20, Karras+21} are used to create mesmerising
high-resolution images.

However, all of these models have a peculiar shortcoming when it comes
to learning and reproducing certain geometric patterns, like those
present in human hands and fingers. For example, Figure~\ref{fig:
  Numbers GANs}b shows the images generated by \DALLE{} 3, when
prompted ``\emph{a realistic human hand showing number \(n\)}'', for
\(n=2, 4\). This phenomenon is universally present in all families of
generative models, from GANs \cite{Goodfellow+14} to denoising
diffusion models \cite{SongME21, HoJA20}, whether they are based on
\emph{convolution} \cite{FukushimaM82, Lecun+89}, \emph{Vision
  Transformers} (\emph{ViT}) \cite{Dosovitskiy+21, Parmar+19}, or a
combination of both \cite{Haiping+21}.


%
\begin{figure}[t]
  \centering
  \begin{minipage}[b]{0.242\linewidth}
    \centering
    \includegraphics[width=\textwidth]{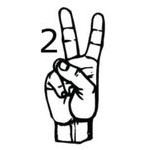}
  \end{minipage}
  \hfill
  \begin{minipage}[b]{0.242\linewidth}
    \centering
    \includegraphics[width=\textwidth]{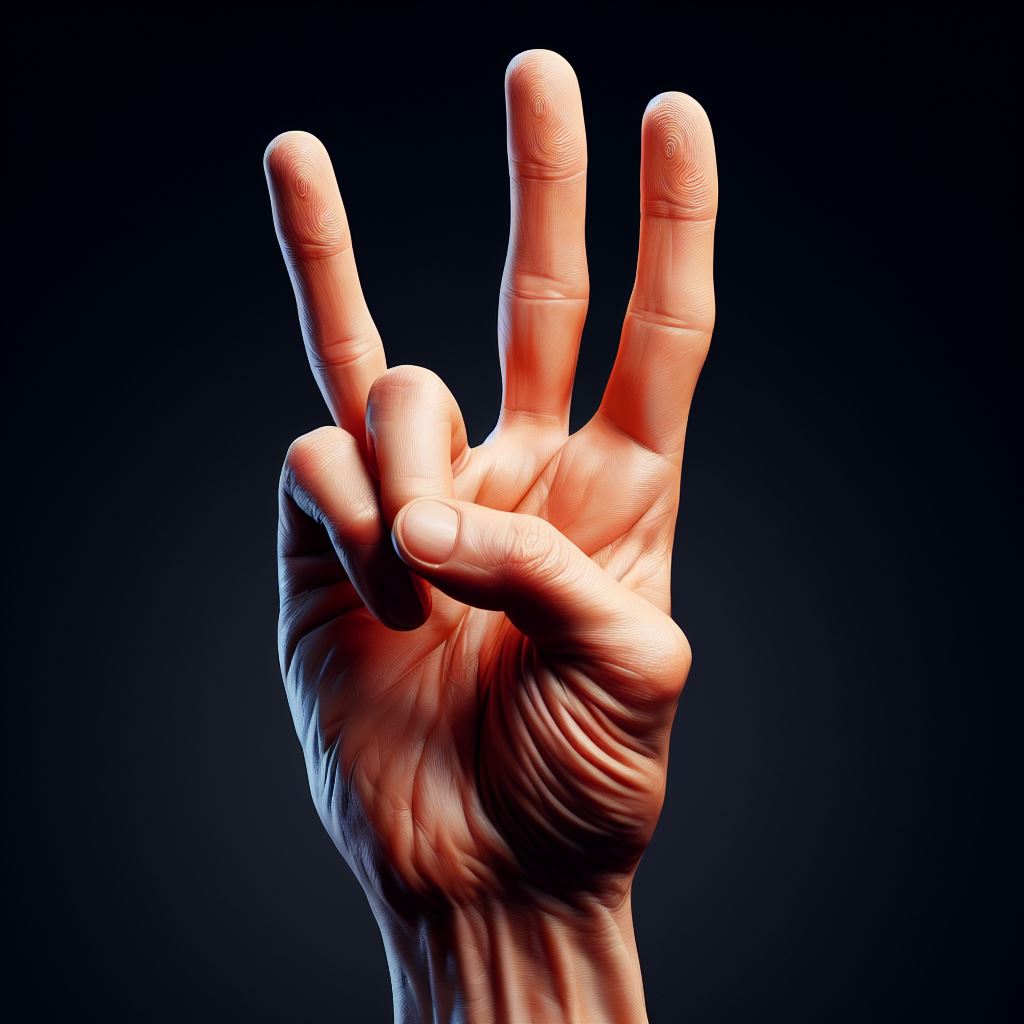}
  \end{minipage}
  \hfill
  \begin{minipage}[b]{0.242\linewidth}
    \centering
    \includegraphics[width=\textwidth]{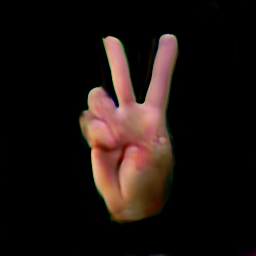}
  \end{minipage}
  \vspace{.15em}
  \hfill
  \begin{minipage}[b]{0.242\linewidth}
    \centering
    \includegraphics[width=\textwidth]{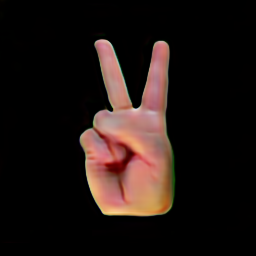}
  \end{minipage}
  \begin{minipage}[b]{0.242\linewidth}
    \centering
    \includegraphics[width=\textwidth]{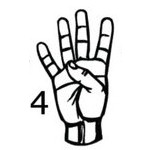}
    \caption*{{\small (a) Hand drawn}}
  \end{minipage}
  \begin{minipage}[b]{0.242\linewidth}
    \centering
    \includegraphics[width=\textwidth]{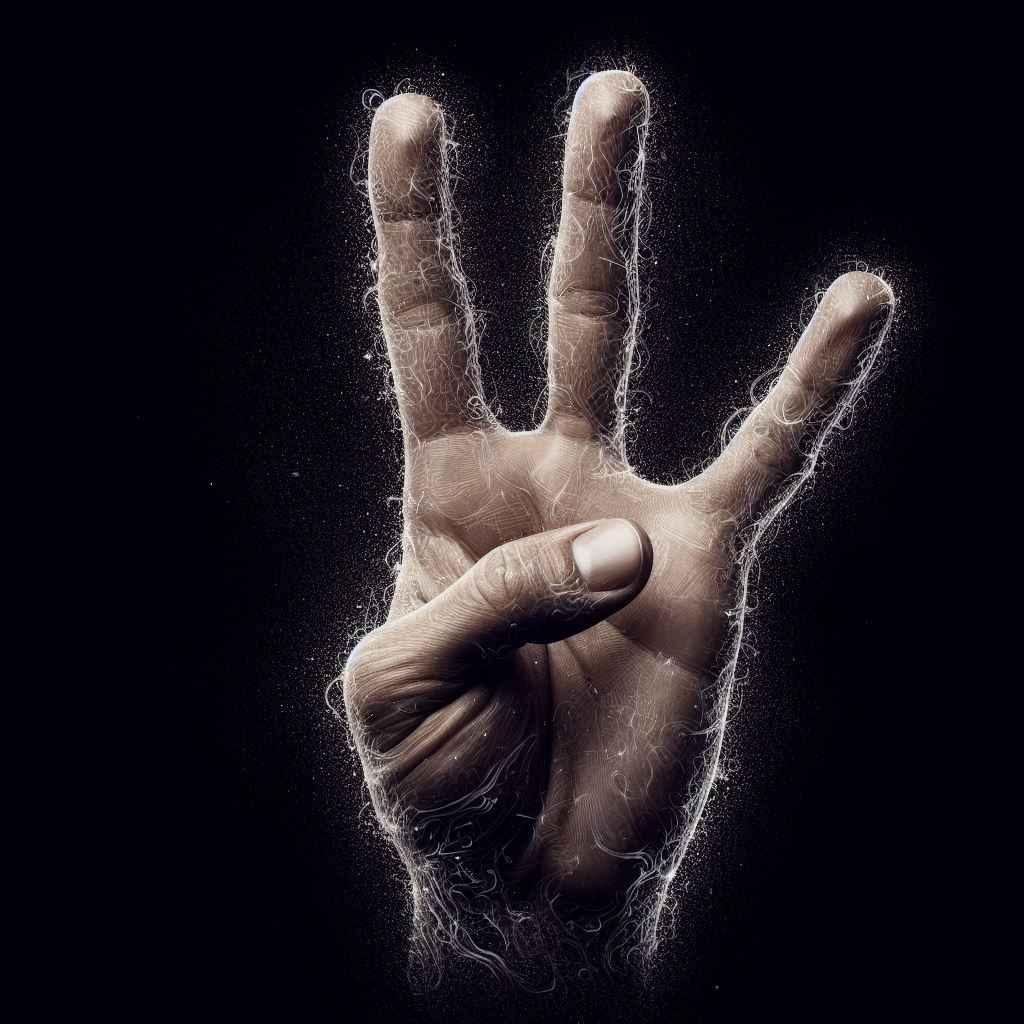}
    \caption*{{\small (b) \DALLE{} 3}}
  \end{minipage}
  \begin{minipage}[b]{0.242\linewidth}
    \centering
    \includegraphics[width=\textwidth]{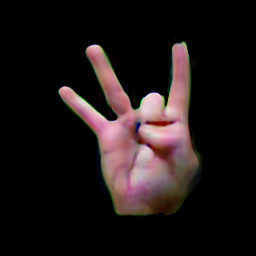}
    \caption*{{\small (c) ConvGAN}}
  \end{minipage}
  \begin{minipage}[b]{0.242\linewidth}
    \centering
    \includegraphics[width=\textwidth]{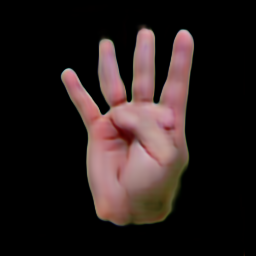}
    \caption*{{\small (d) GeoGAN}}
  \end{minipage}
  \caption{{\small Human hands showing numbers 2 and 4 as drawn by
      hand (Fig.~\ref{fig: Numbers GANs}a) and as generated by
      \DALLE{} 3 (Fig.~\ref{fig: Numbers GANs}b), a standard
      convolutional GAN (Fig.~\ref{fig: Numbers GANs}c), and GeoGAN
      (ours) (Fig.~\ref{fig: Numbers GANs}d). The comparison is
      between ConvGAN and GeoGAN only. Images generated by \DALLE{} 3
      are \textbf{only} included to illustrate the struggles of SoA models in
      generating human hands.}}
  \label{fig: Numbers GANs}
\end{figure}

Human painters, on the other hand, are able to draw flawless pictures
of hands. It is, in part, because, unlike the generative models,
painters know how hands work, providing them with a knowledge of what
hands can and cannot do. Another contributing factor is that human
painters learn how to draw hands by breaking down and simplifying them
into simple geometric shapes, as shown in Figure~\ref{fig: Numbers GANs}a.

Generative models' shortcomings are caused by two contributing
factors, models' design and architecture, as well as the training
dataset and methodology. Taking into consideration that the SoA image
generative models are trained on a vast collection of images on the
Internet and are further enhanced by methods, such as Reinforcement
Learning with Human Feedback (RLHF) \cite{Christiano+17}, the latter
is not the core issue. In the last few years, it has become
evident that the model size corresponds directly to the quality of
generated images, resulting in models that produce hyper-realistic
images with incredible texture and lighting, yet fall short in
generating intricate patterns, indicating the fundamental inability of
these models in learning the geometric representation of human hands
and fingers. To better understand the ability, or lack thereof, of
convolution operation in learning geometric information, we evaluate
its performance on a geometric task, computing the centre of mass of
finitely many points in a 2-dimensional plane.

%
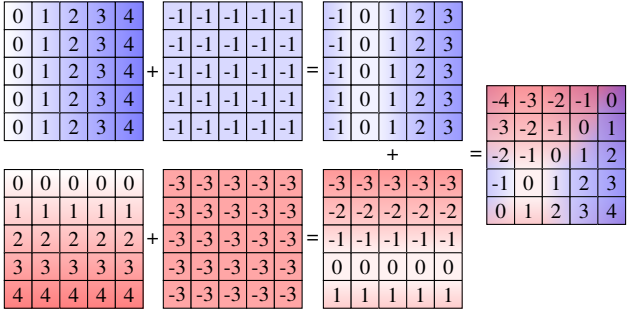
\begin{figure}[t]
  \centering
  \hspace*{-1.5mm}
  \begin{tikzpicture}[
    node distance=.35em and -.55em, font={\footnotesize},
    matrixnode/.style={matrix of nodes,nodes={draw,minimum size=1.05em,anchor=center,inner sep=0pt,outer sep=0pt},
    column sep=-\pgflinewidth, row sep=-\pgflinewidth}
    ]
    \matrix[matrixnode] (x)
    {
      |[right color=blue!10, left color=white]|0 & |[right color=blue!20, left color=blue!10]|1 & |[right color=blue!30, left color=blue!20]|2 & |[right color=blue!40, left color=blue!30]|3 & |[right color=blue!50, left color=blue!40]|4 &\\
      |[right color=blue!10, left color=white]|0 & |[right color=blue!20, left color=blue!10]|1 & |[right color=blue!30, left color=blue!20]|2 & |[right color=blue!40, left color=blue!30]|3 & |[right color=blue!50, left color=blue!40]|4 &\\
      |[right color=blue!10, left color=white]|0 & |[right color=blue!20, left color=blue!10]|1 & |[right color=blue!30, left color=blue!20]|2 & |[right color=blue!40, left color=blue!30]|3 & |[right color=blue!50, left color=blue!40]|4 &\\
      |[right color=blue!10, left color=white]|0 & |[right color=blue!20, left color=blue!10]|1 & |[right color=blue!30, left color=blue!20]|2 & |[right color=blue!40, left color=blue!30]|3 & |[right color=blue!50, left color=blue!40]|4 &\\
      |[right color=blue!10, left color=white]|0 & |[right color=blue!20, left color=blue!10]|1 & |[right color=blue!30, left color=blue!20]|2 & |[right color=blue!40, left color=blue!30]|3 & |[right color=blue!50, left color=blue!40]|4 &\\
    };

    \matrix[matrixnode] (y) [below=of x]
    {
      |[top color=white, bottom color=red!10]|0 & |[top color=white, bottom color=red!10]|0 & |[top color=white, bottom color=red!10]|0 & |[top color=white, bottom color=red!10]|0 & |[top color=white, bottom color=red!10]|0 &\\
      |[top color=red!10, bottom color=red!20]|1 & |[top color=red!10, bottom color=red!20]|1 & |[top color=red!10, bottom color=red!20]|1 & |[top color=red!10, bottom color=red!20]|1 & |[top color=red!10, bottom color=red!20]|1 &\\
      |[top color=red!20, bottom color=red!30]|2 & |[top color=red!20, bottom color=red!30]|2 & |[top color=red!20, bottom color=red!30]|2 & |[top color=red!20, bottom color=red!30]|2 & |[top color=red!20, bottom color=red!30]|2 &\\
      |[top color=red!30, bottom color=red!40]|3 & |[top color=red!30, bottom color=red!40]|3 & |[top color=red!30, bottom color=red!40]|3 & |[top color=red!30, bottom color=red!40]|3 & |[top color=red!30, bottom color=red!40]|3 &\\
      |[top color=red!40, bottom color=red!50]|4 & |[top color=red!40, bottom color=red!50]|4 & |[top color=red!40, bottom color=red!50]|4 & |[top color=red!40, bottom color=red!50]|4 & |[top color=red!40, bottom color=red!50]|4 &\\
    };

    \node (r1p) [right=of x] {+};
    \node (r2p) [right=of y] {+};

    \matrix[matrixnode] (randX) [right=of r1p]
    {
      |[fill=blue!15]|-1 & |[fill=blue!15]|-1 & |[fill=blue!15]|-1 & |[fill=blue!15]|-1 & |[fill=blue!15]|-1\\
      |[fill=blue!15]|-1 & |[fill=blue!15]|-1 & |[fill=blue!15]|-1 & |[fill=blue!15]|-1 & |[fill=blue!15]|-1\\
      |[fill=blue!15]|-1 & |[fill=blue!15]|-1 & |[fill=blue!15]|-1 & |[fill=blue!15]|-1 & |[fill=blue!15]|-1\\
      |[fill=blue!15]|-1 & |[fill=blue!15]|-1 & |[fill=blue!15]|-1 & |[fill=blue!15]|-1 & |[fill=blue!15]|-1\\
      |[fill=blue!15]|-1 & |[fill=blue!15]|-1 & |[fill=blue!15]|-1 & |[fill=blue!15]|-1 & |[fill=blue!15]|-1\\
    };

    \matrix[matrixnode] (randY) [right=of r2p]
    {
      |[fill=red!35]|-3 & |[fill=red!35]|-3 & |[fill=red!35]|-3 & |[fill=red!35]|-3 & |[fill=red!35]|-3\\
      |[fill=red!35]|-3 & |[fill=red!35]|-3 & |[fill=red!35]|-3 & |[fill=red!35]|-3 & |[fill=red!35]|-3\\
      |[fill=red!35]|-3 & |[fill=red!35]|-3 & |[fill=red!35]|-3 & |[fill=red!35]|-3 & |[fill=red!35]|-3\\
      |[fill=red!35]|-3 & |[fill=red!35]|-3 & |[fill=red!35]|-3 & |[fill=red!35]|-3 & |[fill=red!35]|-3\\
      |[fill=red!35]|-3 & |[fill=red!35]|-3 & |[fill=red!35]|-3 & |[fill=red!35]|-3 & |[fill=red!35]|-3\\
    };

    \node (r1e) [right=of randX] {=};
    \node (r2e) [right=of randY] {=};

    \matrix[matrixnode] (coordX) [right=of r1e]
    {
      |[left color=blue!20, right color=blue!05]|-1 &|[left color=blue!05, right color=blue!05]| 0 & |[left color=blue!05, right color=blue!20]|1 & |[left color=blue!20, right color=blue!30]|2 & |[left color=blue!30, right color=blue!40]|3\\
      |[left color=blue!20, right color=blue!05]|-1 &|[left color=blue!05, right color=blue!05]| 0 & |[left color=blue!05, right color=blue!20]|1 & |[left color=blue!20, right color=blue!30]|2 & |[left color=blue!30, right color=blue!40]|3\\
      |[left color=blue!20, right color=blue!05]|-1 &|[left color=blue!05, right color=blue!05]| 0 & |[left color=blue!05, right color=blue!20]|1 & |[left color=blue!20, right color=blue!30]|2 & |[left color=blue!30, right color=blue!40]|3\\
      |[left color=blue!20, right color=blue!05]|-1 &|[left color=blue!05, right color=blue!05]| 0 & |[left color=blue!05, right color=blue!20]|1 & |[left color=blue!20, right color=blue!30]|2 & |[left color=blue!30, right color=blue!40]|3\\
      |[left color=blue!20, right color=blue!05]|-1 &|[left color=blue!05, right color=blue!05]| 0 & |[left color=blue!05, right color=blue!20]|1 & |[left color=blue!20, right color=blue!30]|2 & |[left color=blue!30, right color=blue!40]|3\\
    };

    \matrix[matrixnode] (coordY) [right=of r2e]
    {
      |[top color=red!40, bottom color=red!30]|-3 & |[top color=red!40, bottom color=red!30]|-3 & |[top color=red!40, bottom color=red!30]|-3 & |[top color=red!40, bottom color=red!30]|-3 & |[top color=red!40, bottom color=red!30]|-3\\
      |[top color=red!30, bottom color=red!20]|-2 & |[top color=red!30, bottom color=red!20]|-2 & |[top color=red!30, bottom color=red!20]|-2 & |[top color=red!30, bottom color=red!20]|-2 & |[top color=red!30, bottom color=red!20]|-2\\
      |[top color=red!20, bottom color=red!05]|-1 & |[top color=red!20, bottom color=red!05]|-1 & |[top color=red!20, bottom color=red!05]|-1 & |[top color=red!20, bottom color=red!05]|-1 & |[top color=red!20, bottom color=red!05]|-1\\
      |[top color=red!05, bottom color=red!05]|0 &  |[top color=red!05, bottom color=red!05]|0 &  |[top color=red!05, bottom color=red!05]|0 &  |[top color=red!05, bottom color=red!05]|0 &  |[top color=red!05, bottom color=red!05]|0\\
      |[top color=red!05, bottom color=red!20]|1 &  |[top color=red!05, bottom color=red!20]|1 &  |[top color=red!05, bottom color=red!20]|1 &  |[top color=red!05, bottom color=red!20]|1 &  |[top color=red!05, bottom color=red!20]|1\\
    };

    \node (midPlus) at ($.5*(coordX.center) + .5*(coordY.center)$) {+};

    \node (r1Invisible) [right=of coordX] {};
    \node (r2Invisible) [right=of coordY] {};
    \node (longE) at ($.5*(r1Invisible.east)+.5*(r2Invisible.east)$) {=};

        \matrix[matrixnode] (coordX) [right=of longE]
    { 
      |[upper left=blue!20!red!50,upper right=blue!5!red!42.10, lower left=blue!20!red!37.5,lower right=blue!5!red!31.58]|-4 &
      |[upper left=blue!5!red!42.10,upper right=blue!5!red!42.10, lower left=blue!5!red!31.58,lower right=blue!5!red!31.58]|-3 &
      |[upper left=blue!5!red!42.10,upper right=blue!20!red!50, lower left=blue!5!red!31.58,lower right=blue!20!red!37.5]|-2 &
      |[upper left=blue!20!red!50,upper right=blue!30!red!57.14, lower left=blue!20!red!37.5,lower right=blue!30!red!42.85]|-1 &
      |[upper left=blue!30!red!57.14,upper right=red!40!blue!66.66, lower left=red!30!blue!42.86,lower right=red!30!blue!57.14]|0\\
      |[upper left=blue!20!red!37.5,upper right=blue!5!red!31.58, lower left=blue!20!red!25,lower right=blue!5!red!21.05]|-3 & 
      |[upper left=blue!5!red!31.58,upper right=blue!5!red!31.58, lower left=blue!5!red!21.05,lower right=blue!5!red!21.05]|-2 & 
      |[upper left=blue!5!red!31.58,upper right=blue!20!red!37.5, lower left=blue!5!red!21.05,lower right=blue!20!red!25]|-1 &  
      |[upper left=blue!20!red!37.5,upper right=blue!30!red!42.85, lower left=blue!20!red!25,lower right=red!20!blue!37.5]|0 & 
      |[upper left=red!30!blue!42.86,upper right=red!30!blue!57.14, lower left=red!20!blue!37.5,lower right=red!20!blue!50]|1\\
      |[upper left=blue!20!red!25,upper right=blue!5!red!21.05, lower left=red!5!blue!21.05,lower right=red!5!blue!5.26]|-2 & 
      |[upper left=blue!5!red!21.05,upper right=blue!5!red!21.05, lower left=blue!5!red!5.26,lower right=blue!5!red!5.26]|-1 &  
      |[upper left=blue!5!red!21.05,upper right=blue!20!red!25, lower left=blue!5!red!5.26,lower right=red!5!blue!21.05]|0 &  
      |[upper left=blue!20!red!25,upper right=red!20!blue!37.5, lower left=red!5!blue!21.05,lower right=red!5!blue!31.58]|1 & 
      |[upper left=red!20!blue!37.5,upper right=red!20!blue!50, lower left=red!5!blue!31.58,lower right=red!5!blue!42.10]|2\\
      |[upper left=red!5!blue!21.05,upper right=red!5!blue!5.26, lower left=red!5!blue!21.05,lower right=red!5!blue!5.26!]|-1 &  
      |[upper left=blue!5!red!5.26,upper right=blue!5!red!5.26, lower left=blue!5!red!5.26,lower right=blue!5!red!5.26]|0 &  
      |[upper left=blue!5!red!5.26,upper right=red!5!blue!21.05, lower left=blue!5!red!5.26,lower right=red!5!blue!21.05]|1 &  
      |[upper left=red!5!blue!21.05,upper right=red!5!blue!31.58, lower left=red!5!blue!21.05,lower right=red!5!blue!31.58]|2 & 
      |[upper left=red!5!blue!31.58,upper right=red!5!blue!42.10, lower left=red!5!blue!31.58,lower right=red!5!blue!42.10]|3\\
      |[upper left=red!5!blue!21.05,upper right=blue!5!red!5.26, lower left=blue!20!red!25,lower right=blue!5!red!21.05]|0 &  
      |[upper left=blue!5!red!5.26,upper right=blue!5!red!5.26, lower left=blue!5!red!21.05,lower right=blue!5!red!21.05]|1 &  
      |[upper left=blue!5!red!5.26,upper right=red!20!blue!25, lower left=blue!5!red!21.05,lower right=red!20!blue!21.05]|2 &  
      |[upper left=red!5!blue!21.05,upper right=red!5!blue!31.58, lower left=red!20!blue!25,lower right=red!20!blue!37.5]|3 & 
      |[upper left=red!5!blue!31.58,upper right=red!5!blue!42.10, lower left=red!20!blue!37.5,lower right=red!20!blue!50]|4\\
    };
    
  \end{tikzpicture}
  \caption{{\small A \(5 \times 5\) geometry channel of rank 2 is illustrated
    in the rightmost tensor. The top and bottom rows correspond to
    horizontal and vertical coordinates, respectively. The standard
    horizontal and vertical coordinates are shown in the leftmost
    column. Tensors in the second column show random horizontal and
    vertical shifts. In the implementation, coordinate channels are
    divided by their sizes (in this case 4), and for optimisation, we
    sample horizontal and vertical shifts at once as a single random
    number representing their sum; thus, reducing the number of
    additions and samplings.}}
  \label{fig: GeoConv Illustration}
\end{figure}

In this paper, by providing convolution layers with a single
\emph{Geometry Positional} (\emph{GeoPos}) channel, encoding the Cartesian
coordinates, as presented in \cref{fig: GeoConv Illustration}, we
significantly improve convolution's capabilities. As
illustrated in \cref{fig: GeoConv VAE}, GeoPos is appended to the
convolution's input. We refer to consecutive concatenation of
GeoPos and application of convolution as \emph{GeoConv}. 
Compared to existing approaches, like CoordConv \cite{Liu+18}, GeoConv
\begin{enumerate}
\item as we prove in \cref{thm: Filter Collapse,thm: Equivalence}, is
  computationally optimal and only concatenates a single channel to
  the convolution's input, compared to the \(n\) channels of CoordConv
  for \(n\)-dimensional convolution,
\item allows for random translations (shifts) of the Cartesian coordinate system to avoid learning incorrect absolute positional correlations (cf. \cref{subsec: Positional Dependencies}), and
\item is more robust due to the smoothing effect of random shifts, making it the ideal candidate in generative applications, such as in GANs and VAEs.
\end{enumerate}

Note that the random shifts in the GeoPos channel are
different from those of the input. In particular, we found out contrary to the claim of \cite{Liu+18}, the mere addition of coordinate channels does not prevent mode collapse in GANs, even when we augment its inputs with random transformations. In fact, we found out that CoordConv, which does not incorporate random shifts, is more prone to mode collapse in GANs than even the vanilla convolution.

\begin{figure}[t]
  \centering
  \includegraphics[width=\columnwidth,trim=15 45 4
  0,clip]{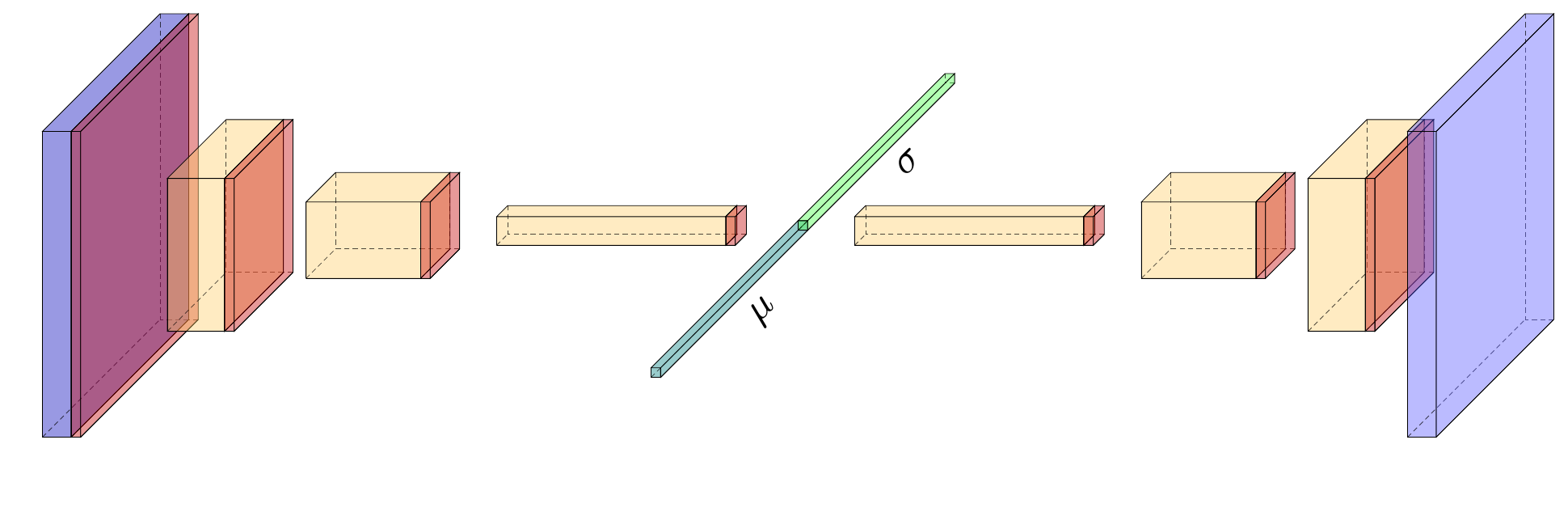}
  \caption{{\small GeoConv in a VAE. Purple blocks indicate the input and
    output tensors, yellow blocks represent the output tensors
    resulting from previous layers' convolution operation, and orange
    blocks indicate the geometry channels appended to them during the
    GeoConv's operation before applying the next convolution.}}
  \label{fig: GeoConv VAE}
\end{figure}

In the rest of \cref{sec: Evaluation}, we show that a
\emph{GeoConv-based GAN} (\emph{GeoGAN}) allows us to generate
realistic hand gestures in the \emph{American Sign Language}
(\emph{ASL}), while a standard \emph{Convolutional GAN} (\emph{ConvGAN}) with the same design, but based on standard
convolution, as well as SoA models, such as \DALLE{} 3 fall short of
achieving the same as shown in \cref{fig: Numbers GANs}. Before
presenting our results on hand gesture synthesis, we evaluate GeoGANs on the
widely used \CelebAHQ. In more details, the experiments in this paper are organised in the following order.
\begin{itemize}
\item We first demonstrate \cref{thm: Positional Dependency,thm: Filter Collapse,thm: Equivalence} in practice on two small geometric experiments in Sections \ref{subsec: Positional Dependencies} and \ref{subsec: Centre of Mass}.
\item In \cref{subsubsec: WGAN-GP Face Generation}, we show that a GeoGAN trained on the \CelebAHQ{} dataset
  \cite{Karras+18} generates more realistic human faces than a
  similar ConvGAN. Moreover,
  while ConvGAN collapses within 250 epochs, GeoGAN remains
  stable throughout the training and produces more diverse images that match the dataset's
  distribution.
\item In \cref{subsubsec: WGAN-GP Face Generation}, we train the
  GANs using the Wasserstein distance \cite{Vaserstein69,
    Kantorovich60, ArjovskyCB17} and gradient penalty \cite{Ishaan+17}
    to prevent mode collapse in the ConvGAN. The resulting models are commonly referred to as
  WGAN-GP. Nevertheless, GeoGAN retains its edge.
\item In \cref{subsubsec: WGAN-GP Hand Generation}, we show that the same GeoGAN trained on the Hand
  Gesture dataset \cite{Massey+11}, generates realistic hand gestures in the American sign language, while the ConvGAN
  struggles to properly generate many of the hand gestures
  (cf. Figures \ref{fig: Numbers GANs}c, \ref{fig: Numbers GANs}d and \ref{fig: WGAN-GP Hands}.)
\item In \cref{subsec: VAE}, we evaluate GeoConv for use in VAEs, since VAEs offer numerical metrics that allow comparing GeoConv to CoordConv and standard convolution quantitatively. We
  repeat the experiments with VAEs, training them on \CelebA{}
 \cite{Liu+15}. The VAE based on GeoConv outperforms other VAEs
  in both image quality and diversity, as well as in achieving smaller losses.
\end{itemize}
\subsection*{Related work}
CNNs have been ubiquitously deployed to achieve superhuman performance
in image classification and object detection \cite{SimonyanZ14a,
  He+16}. More recently, they have been used for image generation
using GANs \cite{Goodfellow+14, RadfordMC15, Karras+18}, VAEs
\cite{KingmaW13, Ramesh+21, Ramesh+22}, and denoising diffusion
models \cite{SongME21, HoJA20}.

In recent years, there has been a surge in the adoption of ViT
\cite{Dosovitskiy+21, Parmar+19}, inspired by the successful adoption
of the attention mechanism \cite{BahdanauCB14} and transformers
\cite{Vaswani+17} in natural language processing. Despite their
tremendous success in vision tasks, recent studies indicate that CNNs
are on par with ViT in both accuracy \cite{Smith+23, Liu+22} and
robustness \cite{Bai+21, PintoTD22}.

CNNs differ from human vision in many ways \cite{Kosiorek+19}. For
example, they are often criticised for their limited receptive field,
preventing them from learning wide-apart features within images
\cite{Luo+16, Kosiorek+19, Dai+17}. Previous research has studied the extent to which CNNs are capable of encoding spatial information and how this spatial information, specifically, absolute positional information can be critical in their performance \cite{IslamJB20, Islam+21}. Some attempts to improve
CNN spatial understanding include augmenting CNNs with transformers
\cite{Hendria+21}, using deformable CNNs \cite{Dai+17}, and augmenting
convolutions with coordinate information \cite{Liu+18}. It is worth
mentioning that similar ideas have been used to improve ViT as well
\cite{XieZF22, HouZF21}; however, these approaches are fundamentally
different from the approach taken here, not because of their focus on
transformers, but mainly because they try to address a different
problem in ViTs.

Liu \textit{et al.} \cite{Liu+18} demonstrated that CNNs also fail in transforming the
spatial representation between input and output. They introduced
\emph{CoordConv} as a solution to this problem of
CNNs. \emph{CoordConv} adds one channel per input dimension to the
convolution's input, called coordinate channel. This has proven to
improve CNNs' performance in an array of tasks
\cite{Liu+18}. CoordConv has since been adopted in an array of
applications \cite{WangBL19, Long+20, ChoiKC20, LeeK19}. Nonetheless,
CoordConv has several drawbacks as we discuss in more details in
Problems \ref{it: Problem 1} and \ref{it: Problem 2} as well as in
\cref{sec: Evaluation}.


%
\section{Geometry-aware convolution}
\label{sec: Theory}
As we discussed in the related work, CoordConv mitigates the limited
receptive field of convolutional layers as well as their inability to learn positional information in images by adding two coordinate
channels, one for each dimension, before applying the convolution
operation. These channels are shown in the two leftmost columns in
\cref{fig: GeoConv Illustration}. CoordConv has shown considerable
improvements compared to convolution in an array of tasks
\cite{WangBL19, Long+20, ChoiKC20, LeeK19}. However, as we show in
this paper, CoordConv has several drawbacks both in theory and
practice. In theory,
\begin{enumerate}
\item\label{it: Problem 1} CoordConv learns absolute positional
  correlations from the dataset, thus, resulting in biased models with
  poor performance in various tasks, while GeoConv learns the
  relative positional correlations when using the random shift
  (cf. \cref{thm: Positional Dependency}), and
\item\label{it: Problem 2} CoordConv is suboptimal (cf. \cref{thm:
    Filter Collapse}), i.e., it introduces \(n \ell s_1 \cdots s_n\)
  learnable parameters for a single \(n\)-dimensional convolution
  operation with kernel size \(s_1 \times \dots \times s_n\) and
  \(\ell\) output channels, instead of GeoConv's
  \(\ell s_1 \cdots s_n\) extra parameters.
\end{enumerate}
As we demonstrate in \cref{sec: Evaluation}, these problems result in subpar performance in practice.

In this section, we introduce the \emph{Geometry-aware Convolution},
or \emph{GeoConv} for short, which not only resolves convolution's
limited receptive field and inability to learn positional information,
but also addresses the aforementioned problems of CoordConv. In
summary, GeoConv works as follows. Given an input tensor of size
\(r_1 \times \dots \times r_n\) with \(k\) channels
\(x \in \RR^{r_1 \times \dots \times r_n \times k}\), we first create
a GeoPos channel \(g \in \RR^{r_1 \times \dots \times r_n}\), encoding the
coordinates as well as a random coordinate shift, similar to the one
in the right most column of \cref{fig: GeoConv Illustration}. Tensor \(g\) is then appended to \(x\) resulting in tensor
\((x, g) \in \RR^{r_1 \times \dots \times r_n \times (k+1)}\), which
is then fed into an \(n\)-dimensional convolution \(f\). To better
understand how GeoConv works, let us begin by describing how it
resolves problems \ref{it: Problem 1} and \ref{it: Problem 2}.

\paragraph{Solution to Problem \ref{it: Problem 1}.} The problem with
adding the raw coordinate channels to the images is that, in addition
to learning the spatial information about the image content, the model
develops correlations between features and where they appear in images
rather than their relative position with respect to one another. This
is a fundamental flaw in most applications. For instance, if due
to the bias in the training dataset a feature mostly appears in a
certain part of the images, the model begins to develop bias for the
position of that feature. Such correlations are undesirable
in most real-world scenarios. For example, when training face
recognition models, the input images or videos are nicely cropped and
the faces are centred in the training set; however, in the real world,
where the model is deployed, this is rarely the case. Thus, it is
more essential for a face recognition model to learn where a person's
facial features are located with respect to each other than where they
are exactly located in the input image or video. In \cref{subsec:
  Positional Dependencies}, we explore this problem of CoordConv and
GeoConv's solution in detail.

Therefore, in GeoConv, we introduce random shifts to coordinate
channels to prevent the model from learning unwanted positional bias,
as formally stated and proven in \cref{thm: Positional
  Dependency}. Random shifts are shown in the second column of
\cref{fig: GeoConv Illustration}. Note that these random shifts are
different from random shifts applied to the input in data
augmentation, e.g., values on the edge of the GeoPos channel are defined
in the same way as the ones in the centre, unlike the input's random
shift, where the values on the edge are defined via some padding. Most
notably, applying random shifts to the input does not prevent mode
collapse in GANs that utilise CoordConv architecture.

\begin{theorem}
  \label{thm: Positional Dependency}
  When using random shift, GeoConv learns the relative positional
  information rather than the absolute positional information, as in
  CoordConv.
\end{theorem}
\begin{proof}
  Let us denote the convolution operator with \(*\). As we prove in
  \cref{thm: Filter Collapse}, we can combine the \(n\) coordinate
  channels of CoordConv to a single channel, similar to GeoConv (but
  with no random shift), without affecting its performance. We denote
  this channel by \(c\) and GeoPos' Channel by \(g\). Now, given
  an input tensor \(x\) of rank \(n\) with \(k\) channels, an
  \(s_1 \times \dots \times s_n\) convolution operator \(f\) on the
  \(k\) input channels, and a single GeoPos channel \(g\) (amounting to a
  total of \(k+1\) channels), we have that
  \begin{equation}
    \label{eq: Convolution Distribution}
    f * (x, g) = f^{(1, \dots, k)} \! * x \ + \ f^{(k+1)} \! * g,
  \end{equation}
  where \(f^{(1, \dots, k)}\) and \(f^{(k+1)}\) denote the first \(k\)
  filters of \(f\) and the last filter of \(f\) corresponding to the
  input and GeoPos channel, respectively. Let \(g' = f^{(k+1)} \! *
  g\). We observe that
  \begin{equation}
    \label{eq: GeoChannel vs CoordChannel}
    \begin{split}
    g'_{j_1, \dots, j_n}
    & = \!\sum_{i_1, \dots, i_n} \! f_i^{(k+1)} g_{j_1+i_1, \dots, j_n+i_n}\\
    & = \!\sum_{i_1, \dots, i_n} \! f_i^{(k+1)} (c_{j_1+i_1, \dots, j_n+i_n} \! + r)\\
    & = \!\sum_{i_1, \dots, i_n} \! f_i^{(k+1)} c_{j_1+i_1, \dots, j_n+i_n} \! + s_1 \cdots s_n r\\
    & = \ f^{(k+1)} \! * c + s_1 \cdots s_n r,
    \end{split}
  \end{equation}
  where \(r\) is a random shift sampled from a uniform distribution in
  GeoConv and \(1 \leq j_{\ell} \leq t_{\ell}\) for
  \(1 \leq \ell \leq n\), with \(t_1 \times \cdots \times t_n\) being
  the input shape. It follows from Equations \eqref{eq: Convolution
    Distribution} and \eqref{eq: GeoChannel vs CoordChannel} that
  \begin{equation}
    \label{eq: GeoConv vs CoordConv}
    f * (x, g) = f * (x, c) + s_1 \dots s_n r.
  \end{equation}
  Hence, \(f * (x, g)\) is equal to \(f * (x, c)\) modulo a random
  number \(s_1 \dots s_n r\). This prevents GeoConv from developing
  unwanted correlations between \(f * (x, c)\) and locations resulting
  in this value, while still allowing it to learn the patterns
  present in \(x\).
\end{proof}

\paragraph{Solution to Problem \ref{it: Problem 2}.} CoordConv adds
one coordinate channel per dimension to the input. Nevertheless, as we
formally state and prove in \cref{thm: Filter Collapse,thm: Equivalence}, this is
unnecessary and inefficient. We prove both results in \cref{sec: Proofs}. Let us first state \cref{thm: Filter Collapse}.

%
\begin{theorem}
  \label{thm: Filter Collapse}
  An \(s_1 \times \dots \times s_n\) convolution filter on the
  \(\ell\)-th coordinate channel \(c^{(\ell)}\) in CoordConv does not
  extract any more information than a
  \(1 \times \dots \times 1 \times s_{\ell} \times 1 \times \dots
  \times 1\) convolution filter.
\end{theorem}

We additionally prove that when \(s_1 s_2 \cdots s_n \geq n (s_1 + s_2 + \cdots + s_n)\), then GeoConv and CoordConv operations are mathematically equivalent.
\begin{theorem}
  \label{thm: Equivalence}
  For a CoordConv layer with \(s_1 \times \dots \times s_n\) filters such that \(s_1 s_2 \cdots s_n \geq n (s_1 + s_2 + \cdots + s_n)\), there exists an equivalent GeoConv layer (without random shift) of the same filter size.
\end{theorem}

Therefore, in GeoConv, we combine all coordinate channels into one by
adding them together, resulting in the GeoPos channel, illustrated in the
rightmost column of \cref{fig: GeoConv Illustration}. The GeoPos channel
is then concatenated to the input channels as demonstrated in
\cref{fig: GeoConv VAE}. By using a single geometry channel instead of
the \(n\) coordinate channels in CoordConv, alongside the random
shift, we achieve superior performance compared to CoordConv while
using \((n-1)\ell\) less filter per convolution, where \(\ell\) is the
number of output channels of the convolution. Consequently, we use
\((n-1)\ell s_1s_2 \cdots s_n\) less learnable parameters. This
provides us with a model that is easier to train, faster, smaller, and
thus, deployable in a wider range of edge devices.

\begin{remark}
  It is important to note that when for some i, \(s_i = 1\), then as stated in \cref{thm: Equivalence} we cannot reduce CoordConv to GeoConv. Nonetheless, there exists a trivial exception to this rule, when the convolution operates on 1-dimensional input and has filter size \(s_1 = 1\); in this case CoordConv (with no shift) are trivially the same.
\end{remark}


%
\section{Evaluation}
\label{sec: Evaluation}
In this section, we evaluate GeoConv on a comprehensive range of
tasks. In \cref{subsec: Centre of Mass}, we evaluate GeoConv
capability in geometric tasks by introducing the centre of mass
benchmark, where GeoConv outperforms convolution and CoordConv by up
to 50\% and 35\%, respectively (cf. \cref{fig: Centre of Mass}). In
\cref{subsec: Positional Dependencies}, we compare all three
architectures on a task for their absolute positional bias on a
simple task consisting of classifying images containing the Greek
numbers I, II, and III. GeoConv and convolution demonstrate the least
bias, while CoordConv has the most. In \cref{subsec: GAN}, we compare
all these architectures for use in GANs. We consider standard GAN
\cite{Goodfellow+14} as well as WGAN-GP \cite{Ishaan+17} for
generating human faces by training on the \CelebAHQ{} \cite{Karras+18}
and hands by training on the Hand Gesture dataset
\cite{Massey+11}. GeoConv generates the most realistic and diverse
faces and hands, while CoordConv collapses in early epochs, performing
even worse than standard convolution.

Finally, we compare all three layers for use in VAEs in
\cref{subsec: VAE}. GeoConv again outperforms others in terms of image
quality and diversity as well as achieving smaller losses on both
train and validation data. We have included the in-depth details of
all experiments in \cref{app: Experimental Setup}.

All of the experiments have been performed in a \emph{GPU-poor}
setting on a computer with 128 GB of RAM and a single GeForce RTX 4090
GPU.

\begin{table}[b]
  \centering
  \begin{tabular}{lccc}
                     & Conv  & CoordConv & GeoConv\\
    \toprule
    Avg. loss        & 274.1 & 218.0     & \textbf{117.1}\\
    Norm. avg. loss  & 0.416 & 0.337     & \textbf{0.246}\\
    \# of best perf. & 8     & 15        & \textbf{26} 
  \end{tabular}
    \caption{Comparison of average, normalised average, and best
    performances of convolution, CoordConv, and GeoConv on the mass
    centre experiment.}
    \label{tbl: Mass Centre Summary}
\end{table}
%

\subsection{Calculating centre of mass}
\label{subsec: Centre of Mass}
In this benchmark, the goal is to compute the centre of mass of
finitely many points in an image. The benchmark consists of datasets
with different densities \(d\). Each dataset consists of images
containing white points on a black canvas, and \(d\) denotes the
percentage of white points in the images in a dataset.

For the ablation study, we consider four different designs with \(i\)
layers and \(j\) filters for \(i, j \in \set{1, 2}\), denoted by
\(i\)x\(j\). Models are trained on datasets with different densities,
\(d \in \set{0.001 \times 3^k : 0 \leq k \leq 6}\), using the
Euclidean norm. Therefore, for each training density \(d\), a total of
\(3 \times 4 = 12\) models are trained. All models are then evaluated
on the test sets with the same density as well as all other densities. We have reported the summary of results in \cref{tbl: Mass Centre Summary}. 

To make the comparison comprehensive, we also computed the normalised
losses (by dividing by the summation over all losses across different
architectures and test and train ratios). We have explained this in
detail in \cref{app: Centre of Mass}. As outlined in \cref{tbl: Mass
  Centre Summary}, GeoConv shows considerable advantage compared to
convolution and CoordConv, outperforming them by 46\% and 57\%,
respectively. Moreover, \cref{fig: Centre of Mass} shows the
normalised losses averaged over all train and test densities for each
architecture and for different number of layers and filters. Again,
GeoConv outperforms convolution and CoordConv in all combinations.

\subsection{Positional dependencies}
\label{subsec: Positional Dependencies}
This experiment is designed to evaluate the (absolute) positional bias
learnt by different architectures. The models are trained on the train
dataset consisting of images of Greek numbers I, II, and III; However,
the distribution of where the numbers are located in the images is
different among the train and test datasets, as described in
\cref{app: Positional Dependencies}. As shown in \cref{tbl: Greek
  Numbers}, CoordConv has the worst performance among all three
architectures, and GeoConv and convolution are on par with each
other. Details of models and training details are included in
\cref{app: Positional Dependencies}

\begin{figure}[t]
    \centering
    \includegraphics[width=\columnwidth]{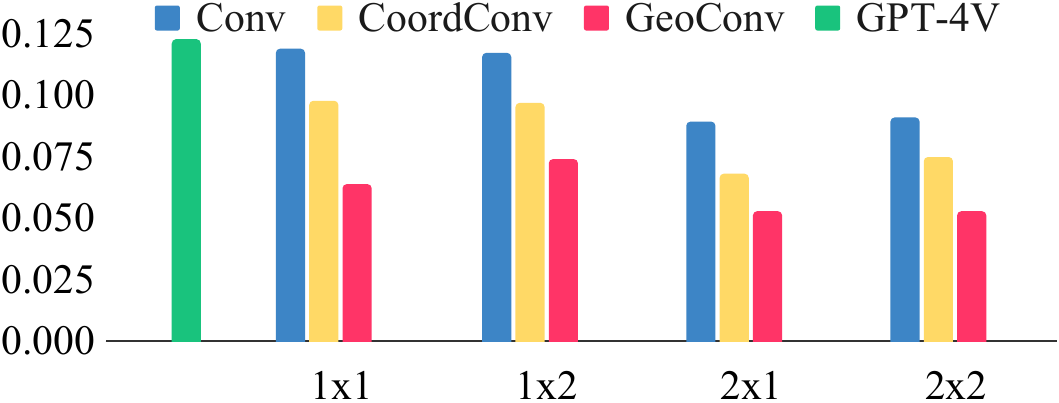}
    \caption{Ablation study on performance of models using each
      architecture with different number of layers and filters. A side
      observation is GPT-4V's \cite{GPT-4,GPT-4V}
      intriguing failure in this task. We evaluated GPT-4V's performance on
      140 (20 per density) dataset images, without
      fine-tuning, but with prompt-engineering, and scaled it by the
      same scaling factor as others.}
    \label{fig: Centre of Mass}
\end{figure}

\begin{table}[b]
  \centering
  \label{tbl: Greek Numbers}
  \begin{tabular}{lccc}
                     & Conv  & CoordConv & GeoConv\\
    \toprule
    Avg. loss        & 1.84 & 2.31     & \textbf{1.59}\\
    Avg. acc. (\%)   & \textbf{34.9} & 34.1     & 34.8\\
    \# of best perf. & 2    & 0         & \textbf{3} 
  \end{tabular}
    \caption{The average loss and accuracy of the models when the
    numbers are moved to all of the possible positions in a
    \(64 \times 64\) canvas}
\end{table}
\begin{figure*}[t]
  \centering
  \begin{minipage}[b]{0.24\textwidth}
    \centering
    \includegraphics[width=\textwidth]{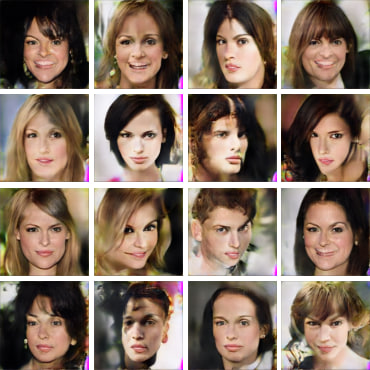}
    \caption*{(a) ConvGAN}
  \end{minipage}
  \hfill
  \begin{minipage}[b]{0.24\textwidth}
    \centering
    \includegraphics[width=\textwidth]{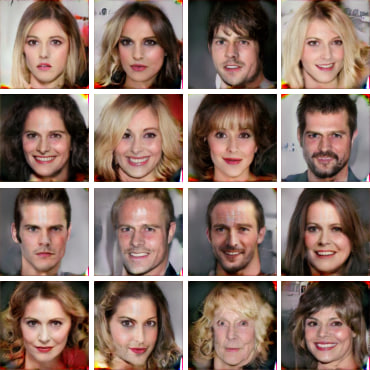}
    \caption*{(b) GeoGAN}
  \end{minipage}
  \hfill
  \begin{minipage}[b]{0.24\textwidth}
    \centering
    \includegraphics[width=\textwidth]{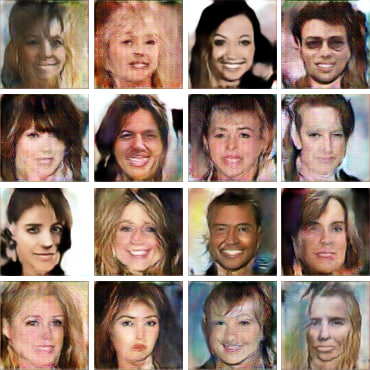}
    \caption*{(c) ConvWGAN-GP}
  \end{minipage}
  \hfill
  \begin{minipage}[b]{0.24\textwidth}
    \centering
    \includegraphics[width=\textwidth]{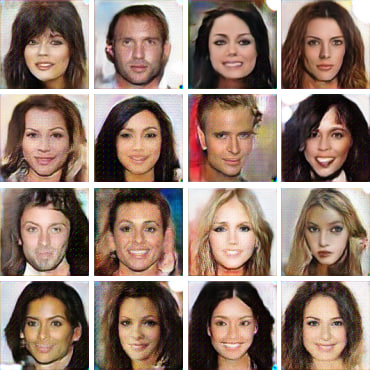}
    \caption*{(d) GeoWGAN-GP}
  \end{minipage}
  \caption{Human faces generated by ConvGAN (\ref{fig:
        GANs}a), GeoGAN (\ref{fig: GANs}b), ConvWGAN-GP (\ref{fig:
        GANs}c), and GeoWGAN-GP (\ref{fig: GANs}d), trained on
      \CelebAHQ{}. Each image is generated as follows. For each of the
      models, we generated 10 images from randomly sampled latent
      points. The image with the highest score from the discriminator
      is added to the canvas. This is repeated 16 times for a
      \(4 \!\times\! 4\) canvas.}
  \label{fig: GANs}
\end{figure*}
%

%
\subsection{Generative adversarial networks}
\label{subsec: GAN}
In this section, we use GeoConv for generating human faces and hand
images using GANs \cite{Goodfellow+14}.  GANs are widely used in an
array of tasks besides the applications considered here, such as
super-resolution \cite{Ledig+17}, photo blending \cite{Wu0ZH19},
etc. and our contribution can open new doors in those applications as
well. For all experiments in this section, we have used the same
design for the models, as described in details in \cref{app: GAN}.
For simplicity, we prepend ``Conv", ``Coord", and ``Geo" prefixes for
the name of the models.  For example, a GAN which uses GeoConv is
referred to as GeoGAN. We have organised this section as follows.

\paragraph{Standard GAN for face generation.} In \cref{subsubsec: GAN
  Face Generation}, we evaluate the performances of the convolution,
GeoConv, and CoordConv in standard GANs \cite{Goodfellow+14,
  RadfordMC15} for generating human faces, by training on the
\CelebAHQ{} dataset \cite{Karras+18} for 450 epochs. CoordGAN
collapses in the first 30 epochs and does not yield meaningful
images. ConvGAN collapses within 250-300 epochs, while GeoConv did not
collapse within 450 epochs. We have provided qualitative and
quantitative summaries of performances in \cref{fig: GANs} and
\cref{tbl: GAN Duels}.

\paragraph{WGAN-GP for face generation.} To prevent mode collapse in
CoordGAN and ConvGAN, we used WGAN-GP \cite{Ishaan+17} with the same
design. This prevented mode collapse in ConvGAN; nevertheless,
CoordGAN collapsed within the first 20 epochs. We also reduced the
number of epochs to 150 since training with gradient, requires
computing second-order derivatives and is computationally expensive;
moreover, we observed that the generated images do not improve after
100 epochs. We have provided a qualitative summary of
performances in \cref{fig: GANs}.

\paragraph{WGAN-GP for hand generation.} We trained conditional
WGAN-GP, with the same design, on the ASL Hand Gesture dataset
\cite{Massey+11} for generating human hand gestures showing numbers
``\emph{0}" to ``\emph{9}" and letters ``\emph{a}" to ``\emph{z}" in
the American sign language for 1,000 epochs. The dataset consists of
2,524 images, with around 70 images per each of the 36 labels. As
expected, CoordWGAN-GP collapses on such a small dataset. Even though
ConvGAN succeeds in generating meaningful hand gestures, it,
sometimes, falls short of reproducing the correct gesture and suffers
in terms of image quality, while GeoConv manages to generate the best
images with correct gesture, as evident in \cref{fig: WGAN-GP Hands}.

\subsubsection{GAN for face generation.}
\label{subsubsec: GAN Face Generation}
Figures \ref{fig: GANs}a and \ref{fig: GANs}b
show the images generated by the ConvGAN and GeoGAN. We have included
the training and models' details, sampling process, as well as more
images generated by each model in \cref{app: GAN}. From Figures
\ref{fig: GANs}a and \ref{fig: GANs}b, we observe
that GeoGAN produces images
\begin{enumerate}
\item\label{it: GAN Face Overall} that are better in terms of the
  overall face layout,
\item have more detail, e.g., teeth, makeup, skin tone, etc., and
\item\label{it: GAN Face Diversity} are more diverse, including 68\%
  female and 31\% male images closely replicating the training set's
  distribution with 63\% female and 37\% male images.
\end{enumerate}
\begin{figure*}[t]
  \centering
  \includegraphics[width=\textwidth]{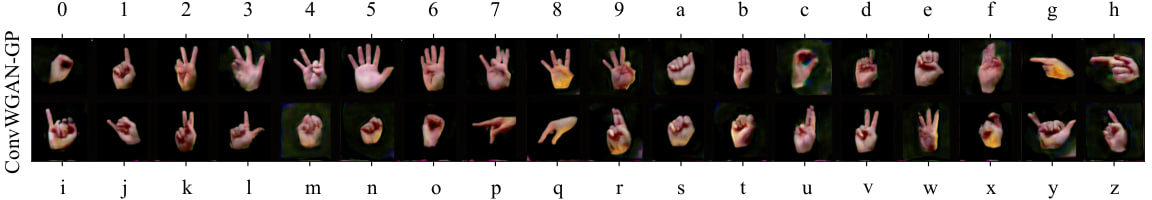}
  \includegraphics[width=\textwidth]{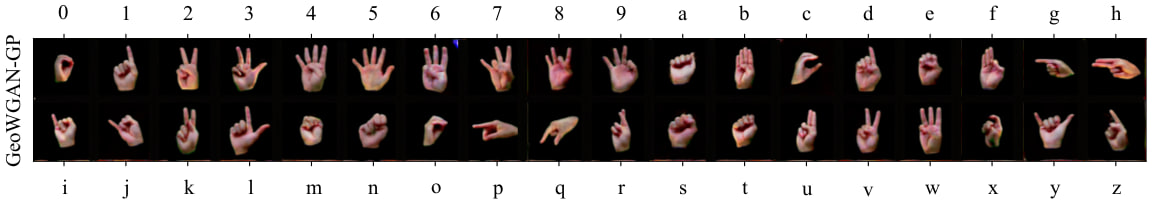}
  \caption{Hand gestures generated by ConvWGAN-GP (top), and
    GeoWGAN-GP (bottom), trained on the ASL Hand dataset. Each image
    is generated as follows. For a given model and label, we generated
    10 images from randomly sampled latent points. The image with
    highest score from the discriminator is added to the canvas. We
    repeat this for each of the 36 labels. Hand gestures generated by
    GeoWGAN-GP, in addition to being clearer, have the correct
    formation and correspond to the correct label, while some of the
    gestures by ConvWGAN-GP, like `\emph{4}', `\emph{6}, `\emph{h}',
    `\emph{r}', and `\emph{s}', show incorrect gestures and some
    other, like `\emph{3}', `\emph{7}', `\emph{c}', `\emph{f}',
    `\emph{i}', and `\emph{o}', are deformed.}
  \label{fig: WGAN-GP Hands}
\end{figure*}
\begin{table}[h!tbp]
  \centering
  \label{tbl: GAN Duels}
  \begin{tabular}{lccc}
              & \multicolumn{3}{c}{Misclassification  Rate (\%)} \\
    \midrule
    Architecture  & Self  & Opp. & Real \\
    \toprule
    ConvGAN        & 75.02 & 7.88 & 0.50 \\
    GeoGAN  & \textbf{42.94}  & \textbf{0.84} & \textbf{0.26}\\ 
  \end{tabular}
    \caption{Duels between ConvGAN and GeoGAN discriminators and
    generators on 10,000 images generated by each of the generators
    and real images from \CelebAHQ{} dataset. Numbers show the
    percentage of images misclassified by each of the discriminators
    against its generator (Self) and opponent's generator
    (Opp). Coordconv is not included due to early mode collapse.}
\end{table}

Additionally, we compared the generator and discriminator of each of
the models against one another and the dataset in \cref{tbl: GAN
  Duels}. GeoGAN's generator deceives ConvGAN's discriminator more by
a factor of 10, and GeoGAN's discriminator is 50\% less likely to
misclassify real images.

%
%

%
\subsubsection{WGAN-GP for face generation.}
\label{subsubsec: WGAN-GP Face Generation}
\emph{Wasserstein GANs} \cite{ArjovskyCB17} \emph{with Gradient
  Penalty} (\emph{WGAN-GP}) \cite{Ishaan+17} emerged as a solution to
the mode collapse problem in standard GANs. In hopes of addressing
mode collapse in ConvGAN and CoordGAN, we trained the models with the
same designs as those in \cref{subsubsec: GAN Face Generation} on the
\CelebAHQ{} dataset. We have included the training detail in
\cref{app: GAN}.  CoordWGAN-GP again failed to produce meaningful
results due to early mode collapse. However, ConvWGAN-GP succeeded in
generating more diverse images, despite falling short in comparison to
GeoWGAN-GP as evident in Figures~\ref{fig: GANs}c and \ref{fig: GANs}d. The
qualities of images generated by both models slightly decreased
compared to the standard GANs. Nonetheless, GeoWGAN-GP still produced
better images compared to ConvWGAN-GP in terms of Items \ref{it: GAN
  Face Overall}-\ref{it: GAN Face Diversity} above.

\subsubsection{WGAN-GP for hand generation.}
\label{subsubsec: WGAN-GP Hand Generation}
\cref{fig: WGAN-GP Hands} shows hand gestures generated by ConvWGAN-GP
and GeoWGAN-GP. Training details and more images are included in
\cref{app: GAN}. ConvWGAN-GP fails to learn the correct
representations for some of the gestures that require intricate
geometric understanding, such as in `\emph{r}', where the middle and
index fingers are crossed. It also generates mutated and contorted
fingers for some other gestures, such as when a finger is hidden
behind another as in `\emph{o}' or `\emph{c}'. This shows standard
convolution's inherent inability to learn complex details. GeoWGAN-GP,
on the other hand, learns more accurate representations for different
hand gestures and generates images of higher quality clear of
mutations and contortions.

\begin{figure*}[t]
  \centering
  \begin{minipage}[b]{0.33\textwidth}
    \centering
    \includegraphics[width=\textwidth]{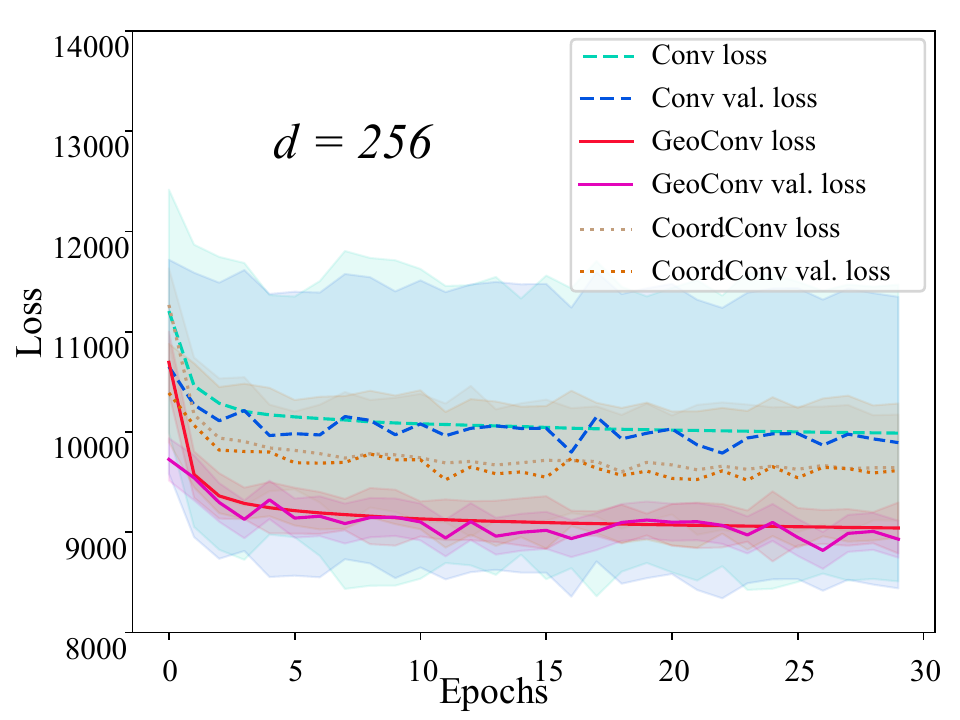}
  \end{minipage}
  \hfill
  \begin{minipage}[b]{0.33\textwidth}
    \centering
    \includegraphics[width=\textwidth]{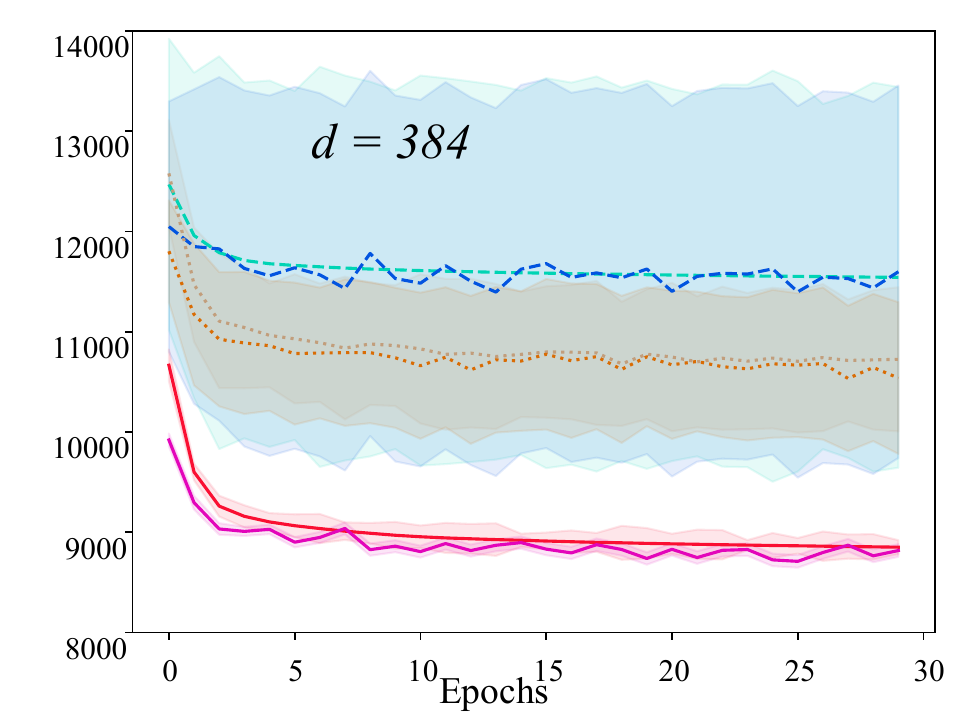}
  \end{minipage}
  \hfill
  \begin{minipage}[b]{0.33\textwidth}
    \centering
    \includegraphics[width=\textwidth]{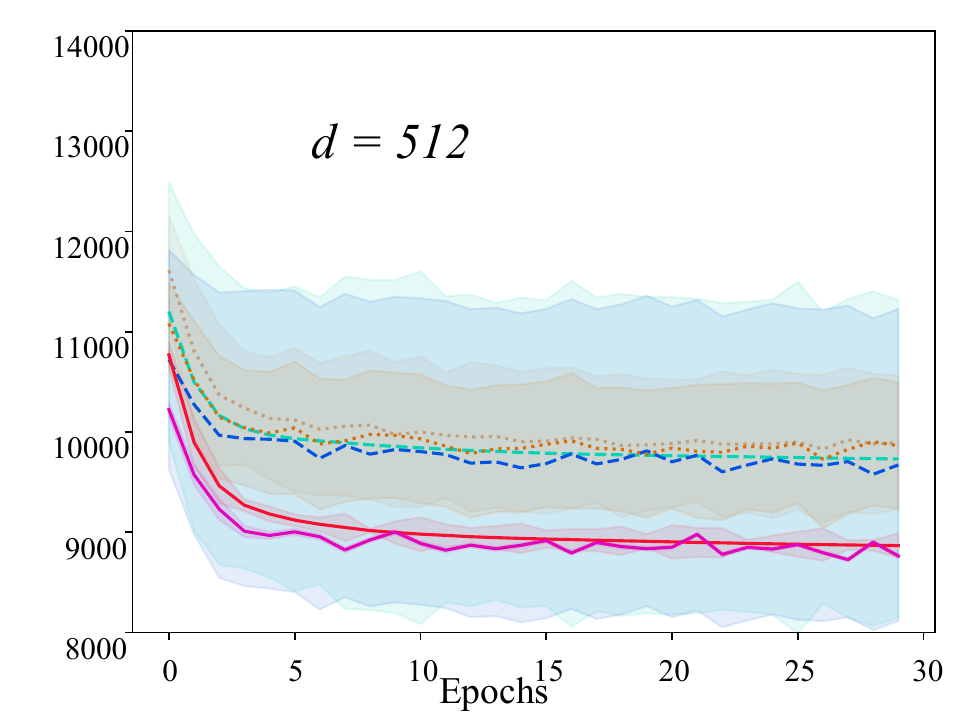}
  \end{minipage}
  \caption{Mean and 95\% CI of train and validation losses of GeoVAE
    (red lines), CoordVAE (dotted brown lines), and ConvVAE (dashed
    blue lines), trained on \CelebA{} dataset for latent dimensions
    \(d \in \{256, 384, 512\}\) over five runs with seeds
    \(0, \dots, 4\). GeoVAE is more consistent across all runs and
    latent dimensions and obtains smaller mean loss and validation
    loss than both ConvVAE and CoordVAE.}
  \label{fig: VAE Evaluation}
\end{figure*}
\subsection{Variational autoencoders}
\label{subsec: VAE}
Due to challenges in quantitative comparison of GANs, we also
evaluate GeoConv for use in VAEs \cite{KingmaW13}, especially since the
effectiveness of convolutions in VAE applications relies on learning
both local and global features from images \cite{Liu+18,
  Kosiorek+19}. VAEs are used in a range of applications
\cite{Kajino19, SinghO22, KovenkoB20}; however, in this section, we
only focus on generating human faces by training on \CelebA{}
dataset \cite{Liu+15}. \cref{app: VAE} includes a similar experiment for
generating hand gestures for ASL numbers and letters, similar to
\cref{subsubsec: WGAN-GP Hand Generation}, as well as model and
training details.

For each latent dimension \(d=256, 384, 512\), we trained GeoVAE,
CoordVAE, and ConvVAE five times to obtain the means and 95\%
\emph{Confidence Intervals} (\emph{CI}) in \cref{fig: VAE
  Evaluation}. Across different latent sizes, GeoVAE obtains 10-25\%
smaller loss and validation loss. Another notable observation is that,
unlike ConvVAE and CoordVAE, GeoVAE's loss does not fluctuate, and the
95\% CI is quite small, especially compared to ConvVAE. We predict
that this may be due to the smoothing effect of the random shift in
GeoConv.

\begin{figure}[h!tbp]
  \centering
  \includegraphics[width=\columnwidth]{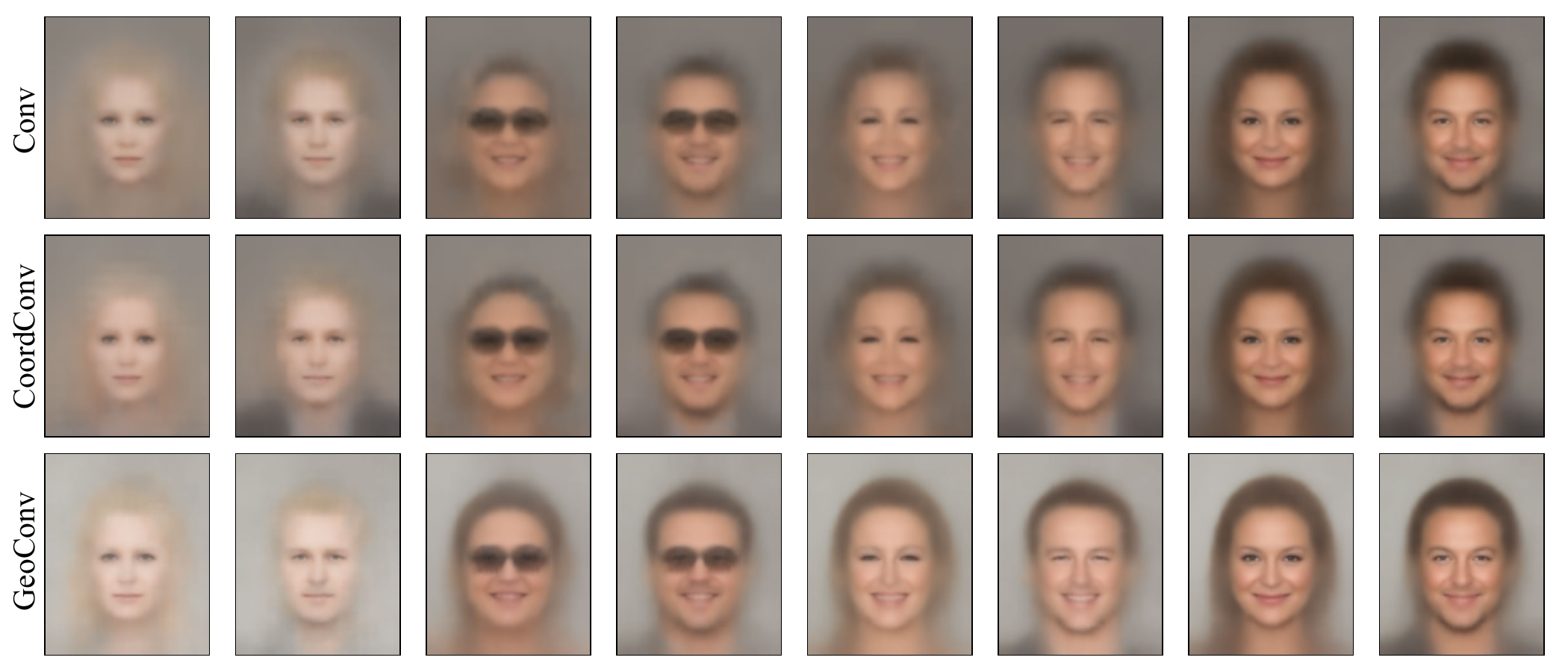}
  \caption{ Images generated by each VAE for different labels.
    Images generated by GeoVAE (bottom) are clearer, have sharper
    edges, and contain more details than those generated by ConvVAE
    (top) and CoordVAE (middle).}
  \label{fig: VAE Images by Label}
\end{figure}

Images generated by VAEs for different labels and latent
values after 30 epochs are shown in Figures \ref{fig: VAE Images by
  Label} and \ref{fig: VAE Images by Latent}, respectively. GeoVAE
demonstrates a notable capacity to produce diverse images given
different latent points. In stark contrast, ConvVAE and CoordVAE fail
to capture the dataset's diversity, generating similar
outputs for all latent points. GeoVAE also exhibits adaptability in
attributes like hairstyle, eye and eyebrow styles, and even skin
tones. Conversely, other models exhibit limited flexibility, yielding
less diverse images. Furthermore, GeoVAE consistently produces
higher-resolution images for all labels and latent points imbued with
more pronounced and distinctive features compared to ConvVAE and CoordVAE generations.

\begin{figure}[t]
  \centering
  \includegraphics[width=\columnwidth]{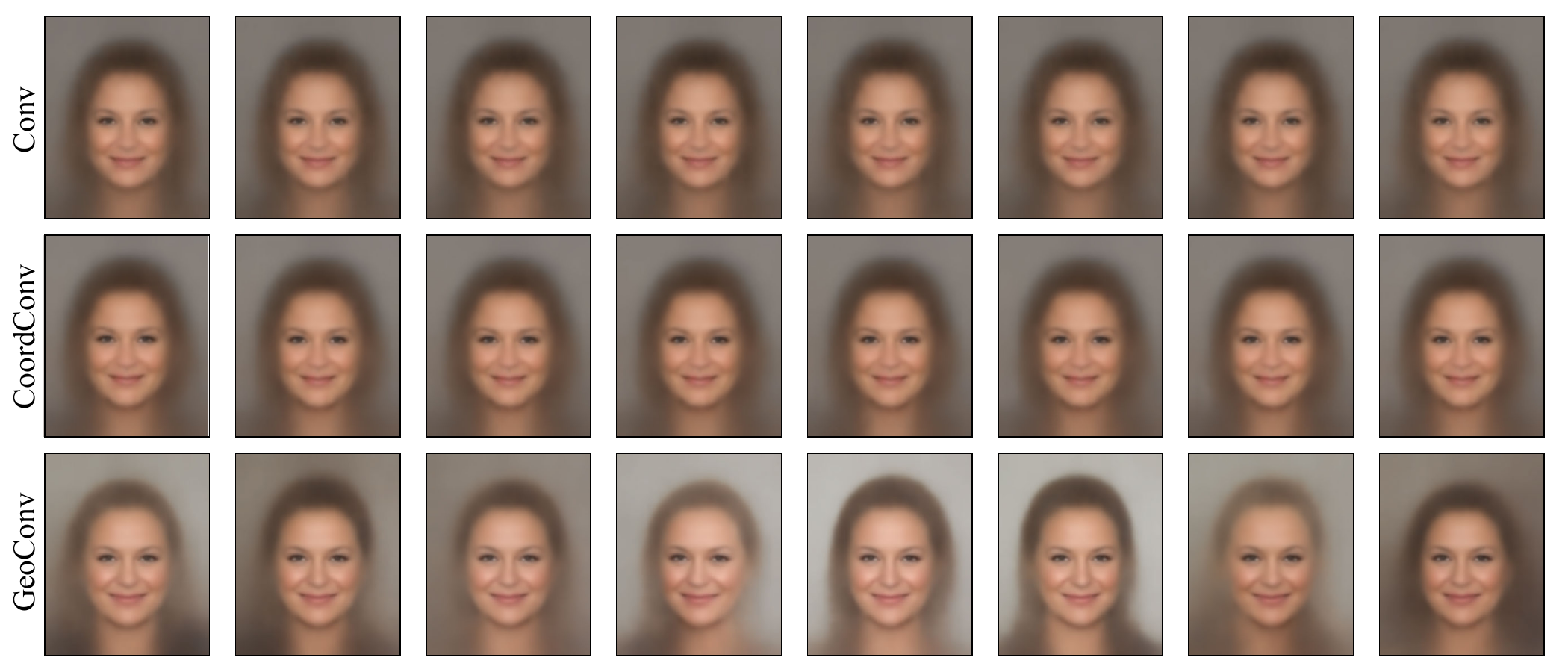}
  \caption{Images generated by each of the VAEs for different random
    latent points. Images generated by GeoVAE (bottom) are more
    diverse and vary with the latent, while images generated by
    ConvVAE (top) and CoordVAE (middle) remain untouched.}
  \label{fig: VAE Images by Latent}
  \vspace{-1em}
\end{figure}
\subsection{Monocular Depth Estimation}
We evaluate convolution, GeoConv, and CoordConv for monocular depth estimation, which also requires learning fine-grained geometric details. We trained three U-Net models on DIODE dataset \cite{Vasiljevic+19-Diode}, (using the three architectures) and as expected the results indicated superior performance in GeoConv and CoordConv compared to pure convolution (in terms of achieving lower validation loss), with GeoConv and CoordConv performing similarly, even though GeoConv achieves this with fewer parameters and computation.

\begin{table}[ht]
  \centering
  \label{tbl: Depth Estimation}
  \begin{tabular}{lccccc}
  Arch.     & Total & SSIM & Smooth. & L1 & L2\\
  \midrule
  Conv.     & 0.097 & 0.191 & 0.003 & 0.155 & 0.037 \\
  CoordConv & 0.094 & 0.187 & 0.003 & 0.152 & 0.035 \\
  GeoConv   & 0.093 & 0.186 & 0.001 & 0.151 & 0.034
  \end{tabular}
    \caption{The validation losses of a standard U-Net model from \href{https://keras.io/examples/vision/depth_estimation/}{keras.io} using different convolutions for monocular depth estimation (on normalised log-depth maps) trained on DIODE dataset. GeoConv and CoordConv perform better than standard convolution. Surprisingly, GeoConv slightly outerperforms CoordConv despite having fewer parameters. The ``Total'' loss is the average of SSIM, Smoothness, L1, and L2 losses.}
    \vspace{-1em}
\end{table}
%


%
\section{Discussion}
\label{sec: discussion}
In \cref{subsec: GAN}, we observed that GeoGANs generate more diverse
images that match training data's distribution compared to
ConvGANs. In the same way, we observe in \cref{subsec: VAE}, that
GeoVAEs show more variation in generating human faces for different
labels and latent points compared to CoordVAEs and ConvVAEs. Even
though the better performance of GeoConv models is expected, it
remains unclear and requires further investigation why and how GeoConv
models create more diverse images.

Another significant observation from \cref{fig: VAE Evaluation}, is
the remarkable consistency of GeoVAEs' loss curves across different
runs and latent dimensions. Intuitively, we expected GeoVAEs to
outperform their counterparts, but how this led to 5 and 11 times
smaller 95\% CI, in comparison to CoordVAEs and ConvVAEs, requires
additional exploration.


%
\section{Conclusions and future directions}
\label{sec: Conclusions}
In this paper, we demonstrated GeoConv's capabilities in consistently
producing better images with more details and diversity compared to
existing convolutional architectures. We showed this for GANs and VAEs
in generating hand gestures and human faces. Given that diffusion
models suffer from some of the same problems, in particular in
generating human hands, GeoConv provides a promising research avenue
to pursue in the future.

Given the promising performance of GeoConv in the models considered
here, we foresee it can improve large-scale SoA models, which we could
not investigate due to computational constraints. We plan to
investigate this further in our future work. Other avenues of research
that we foresee GeoConv will contribute to include geometric tasks,
such as depth estimation, object segmentation, 3D reconstruction,
video generation, and several other applications.


%
\section*{Acknowledgements}
This work is partially supported by the UK EPSRC via the Centre for
Doctoral Training in Intelligent Games and Game Intelligence (IGGI;
EP/S022325/1) and REXASI-PRO Horizon Europe project
(10.3030/10107002). We thank Vultr Cloud and Google Cloud for providing parts of the computational resources for the experiments. We also thank Zahraa Al Sahili for providing
feedback on previous versions of this work.


{\small
\bibliographystyle{wacv}
\bibliography{bibliography}
}

\newpage
\appendix
\clearpage

\section{Limitations}
\label{sec: limitations}
\paragraph{Technical limitations.} The generative models proposed and studied here are limited in size due to GPU constraints. The experiments included are meant to show the efficacy and efficiency of the proposed GeoConv. We believe they effectively illustrate the fundamental advantages of our approach. However, future work with more substantial computational resources could explore the scalability and performance of this framework in larger, more complex settings.

\paragraph{Societal impacts.} While our experiments demonstrate GeoConv's advantage in small to medium-scale applications, its potential efficacy in larger-scale implementations and in making generated images more realistic and detailed, particularly in areas such as accurate hand postures, could facilitate the creation of more convincing deepfakes. This underscores the need for robust watermarking techniques to mitigate potential misuse and ensure digital content's authenticity.

\section{Experimental setup}
\label{app: Experimental Setup}
In this appendix, we explain the experimental setup in
this paper and provide more images, figures, and tables.

\subsection{Centre of mass}
\label{app: Centre of Mass}
Our motivation for choosing this task and configuration is that it
requires the models to have a good understanding of the locations of a
varying number of points spread out in a 2-dimensional plane with a
few convolutional layers and filters. Therefore, the models need to
obtain a geometric and global knowledge of where the points are,
rather than a local knowledge provided by standard convolutions.
\begin{table*}[ht]
  \centering
  \label{tbl: Centre of Mass 1x1}
  \begin{tabular}{llrrrrrrr}
    & & \multicolumn{7}{c}{Test ratio}\\
    \cmidrule(lr){3-9}
    Train ratio & Architecture & 0.001   & 0.003   & 0.01    & 0.03    & 0.1     & 0.3      & 0.9     \\
    \toprule
    & GeoConv     & 2.581 & \textbf{3.598} & \textbf{18.87} & \textbf{79.65} & \textbf{296.7} & \textbf{916} & \textbf{2777} \\
    0.001 & CoordConv   & \textbf{2.267} & 4.630 & 27.02 & 107.5 & 392.9 & 1208  & 3654 \\
    & Conv & 2.438 & 4.622 & 24.88 & 100.7 & 370.3 & 1140  & 3449 \\
    \midrule
    & GeoConv     & 5.435 & 2.640 & \textbf{2.871} & \textbf{6.01} & \textbf{20.12} & \textbf{66.90} & \textbf{211.4} \\
    0.003 & CoordConv   & 4.530 & 2.112 & 3.553 & 18.33 & 76.43 & 242.8 & 742.3 \\
    & Conv & \textbf{4.356} & \textbf{2.104} & 4.025 & 21.25 & 87.28 & 276.4 & 844.0  \\
    \midrule
    & GeoConv     & 6.381 & 3.180 & 1.291 & 4.558  & 24.15 & 80.95 & 251.6  \\
    0.01 & CoordConv   & 9.380 & 4.875 & 1.971 & \textbf{2.978} & \textbf{8.45} & \textbf{14.17} & \textbf{15.1} \\
    & Conv & \textbf{6.329} & \textbf{3.145} & \textbf{1.261} & 4.608 & 24.48 & 82.07 & 255.1 \\
    \midrule
    & GeoConv     & 9.36 & 5.008 & 2.142 & 1.095  & \textbf{1.495} & \textbf{4.72} & \textbf{13.14} \\
    0.03 & CoordConv   & 11.15 & 6.370  & 2.803 & 1.145  & 4.580 & 11.74 & 14.90 \\
    & Conv & \textbf{6.84} & \textbf{3.668} & \textbf{2.133} & \textbf{0.890} & 7.321 & 29.40 & 95.90 \\
    \midrule
    & GeoConv     & \textbf{7.371} & \textbf{3.948} & 2.405 & 1.925 & 0.610 & 5.696 & 23.21 \\
    0.1 & CoordConv   & 7.548 & 4.042 & 2.377 & 1.837 & \textbf{0.601}  & 5.216 & 21.29 \\
    & Conv & 8.369 & 4.426 & \textbf{2.164} & \textbf{1.398} & 0.667 & \textbf{2.916} & \textbf{12.06} \\
    \midrule
    & GeoConv     & \textbf{6.942} & \textbf{3.859} & 2.875 & 2.988 & 2.440  & 0.350 & 7.430 \\
    0.3 & CoordConv   & 7.874 & 4.163 & 2.279 & 1.895 & 1.506  & \textbf{0.342} & 4.474 \\
    & Conv & 9.035 & 4.841 & \textbf{2.231} & \textbf{1.300} & \textbf{0.789} & 0.348 & \textbf{1.467} \\
    \midrule
    & GeoConv     & 9.228 & \textbf{5.114} & \textbf{2.61} & \textbf{1.75} & \textbf{1.26} & \textbf{0.888} & 0.349 \\
    0.9 & CoordConv   & \textbf{5.176} & 5.655 & 7.66 & 8.44 & 8.07 & 6.095 & \textbf{0.147} \\
    & Conv & 5.221 & 7.904  & 10.67   & 11.39  & 10.77  & 8.085  & 0.156 \\
  \end{tabular}
    \caption{The detailed loss table for 1x1 models in \cref{fig: Centre of Mass}.}
\end{table*}

\paragraph{Dataset details} To cover different scenarios and have a
comprehensive comparison between the architectures, we trained the
networks on 7 synthesised datasets, each containing \(100,000\)
images, of size \(32 \times 32\) with point density \(d\), where
\(d \in \sD = \set{0.001, 0.003, 0.01, 0.03, 0.1, 0.3, 0.9}\).  The
set \(\sD\) is roughly defined by the geometric progression
\(0.001 \times 3^k\) for \(0 \leq k \leq 6\), and covers a varying
range of points starting from \(0.001\) and increasing geometrically
with a factor of roughly \(3\), up to \(0.9\) density.

We then evaluated the performances of each of the networks on 7 test
datasets, each containing \(20,000\) images with density
\(d \in \sD\). All of the networks were trained using the Euclidean
distance, between the predicted mass centre and the true mass centre,
as the loss function. We have included the detailed results for the
base models, i.e., models with 1 convolution layer and 1 filter (this
is shown as 1x1 in \cref{fig: Centre of Mass}) in \cref{tbl: Centre of
  Mass 1x1}. Detailed results for other models can be easily obtained
by running the provided code.

\paragraph{Model design} All networks use convolution layers with a
kernel size of \(3\) and a stride of \(2\) with ReLU activation,
combined with a dense output layer with \(2\) nodes, corresponding to
the \(x\) and \(y\) coordinates of the mass centre. As an ablation
study, we consider 4 networks \(i\)x\(j\), where
\(0 \leq i, j \leq 2\).



%
\subsection{Positional dependencies}
\label{app: Positional Dependencies}
\paragraph{Detaset details}This dataset is designed to evaluate
positional bias, described in \cref{subsec: Positional Dependencies},
in vision models. This dataset includes \(64 \times 64\) images
containing Greek numbers I, II, and III, corresponding to labels 1, 2,
and 3. In the training set, the Greek numbers are almost centred in
the image, with little horizontal and vertical shifts, while in the
test sets the Greek numbers move farther from the centre the images.

\paragraph{Model Design} We consider convolutional models with varying
number of layers ranging over \(1, \dots, 5\). The \(n\)-th layer in
each model has \(2^{n-1}\) filters. All convolutions layers have
kernel size of \(3\) and use a stride of \(2\), with ReLU
activation. The only other layer, is the output layer, which is a
dense layer of size 3. The models are trained on the training set
using the categorical cross entropy loss, which is the standard choice
for multi-class classification tasks.

\begin{table*}[h!]
  \centering
  \label{tbl: Greek Numbers Full}
  \begin{tabular}{llrrrrr}
    Metric & Architecture                & 1 Layer        & 2 Layers & 3 Layers  & 4 Layers & 5 Layers   \\
    \toprule
    \multirow{ 4}{*}{Loss}  & Conv2D &   \textbf{1.16}        & 1.26     & 1.60 & 2.20 & 2.95 \\
                & CoordConv & 1.78          & 2.35     & 2.45 & 2.06 & 2.91\\
                    & GeoConv & 1.23          & 1.63     & \textbf{1.50} & \textbf{1.59} & \textbf{2.04} \\
                    & CoordConv + Shift & 1.43 & \textbf{1.20} & 1.55 & 1.90 & 2.20 \\
    \midrule
     \multirow{ 4}{*}{Acc. (\%)}   & Conv2D & \textbf{36.82} & \textbf{35.25}     &  34.06    & 34.00 & \textbf{34.40}    \\
                    & CoordConv & 34.08 & 34.03     & 34.35     & 34.17 & 34.10    \\
                    & GeoConv & 34.93 & 34.29     & \textbf{36.20}     & \textbf{34.35} & 33.76   \\
                & CoordConv + Shift & 34.70 & 34.24 & 34.35 & 33.81 & 33.27 \\    
  \end{tabular}
    \caption{The average loss and accuracy of the models when the
    numbers are moved to all of the possible positions in a
    \(64 \times 64\) canvas. In addition to the common baselines in all other experiments, we also studied CoordConv with positional shift to see the impact of adding positional shift to CoordConv. }
\end{table*}

As you can see in \cref{tbl: Greek Numbers}, despite having the
highest number of learnable parameters, CoordConv has the worst
performance amongst all the architectures due to the positional bias
learnt during the training. The complete results that \cref{tbl: Greek
  Numbers} is derived from is available in \cref{tbl: Greek Numbers
  Full}.

\subsection{GAN}
\label{app: GAN}
In this section of the appendix, we discuss the details of the
experiments in \cref{subsec: GAN}.

\subsubsection{Dataset details}

\paragraph{\CelebAHQ{} dataset}
\CelebAHQ{} \cite{Karras+18}, introduced in 2018, is a dataset
consisting of 30,000 human face images with \(1024 \!\times\! 1024\)
resolution. Since its introduction, it has been widely used in various
applications for generating realistic human faces. Unlike \CelebA{}
dataset, \CelebAHQ{} does not include annotations on facial
features. This dataset includes 18,943 (63.15\%) female images and
11,057 (36.85\%) male images \cite{Na21}. We use this dataset in our
GAN experiments to gain insights on the capabilities and limitations
of models using different convolution architectures.

\paragraph{ASL Hand Gesture dataset} \label{par: app ASL dataset} ASL
Hand Gesture \cite{Massey+11} is a small dataset consisting of 2,524
annotated hand gesture images representing numbers `\emph{0}' to
`\emph{9}' and English alphabets `\emph{a}' to `\emph{z}' in the
American sign language. The dataset images are almost equally
distributed between all the 36 labels; there are approximately 70
images per each labels. All images are on a black background and of
different sizes, which are resized to \(256 \!\times\! 256\)
resolution at preprocessing.


%
\subsubsection{GANs for for generating face images}
\label{Par: GAN-Face Model Design}

\paragraph{Model design}
The generator and discriminator are designed according to common
practices in training GANs. The discriminator's architecture is
similar to VGG-13. Here, we discuss the generator architecture. In the
generator, after one dense layer, and a reshape layer that takes the
1-dimensional latent to a 3-dimensional tensor, we have 5 blocks of
layers, each consisting of the following 3 layers:
\begin{itemize}
\item A transposed convolution/GeoConv/CoordConv layer with a stride
  of 2 and kernel size of 3 with no padding and leaky ReLU activation.
\item A convolution/GeoConv/CoordConv layer with a stride of 1 and
  kernel size of 3 with leaky ReLU activation.
\item A batch normalization layer.
\end{itemize}
In the end, the output layer of the generator is a
convolution/GeoConv/CoordConv layer with the same specification as
before except for the activation which is sigmoid.

\paragraph{Training detail}
For training the models in this experiment, we use the binary
cross-entropy loss which is the common method for training GANs. We
trained each model for 500 epochs. After reaching 400 epochs, none of
the models showed any improvements.

\paragraph{A closer look at generated images}
\cref{fig: App GANs} portrays a \(6\!\times\!6\) canvas with more
images of ConvGAN and GeoGAN. Notice the quality, colour, and
diversity of the images by each of the models.

\begin{figure*}[t]
  \centering
  \begin{minipage}[b]{0.48\textwidth}
    \centering
    \includegraphics[width=\textwidth]{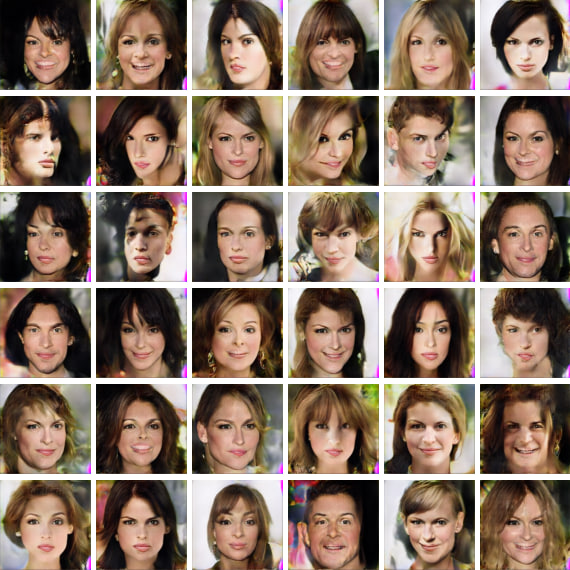}
     \caption*{{\small (a) ConvGAN}}
  \end{minipage}  
  \hfill
  \begin{minipage}[b]{0.48\textwidth}
    \centering
    \includegraphics[width=\textwidth]{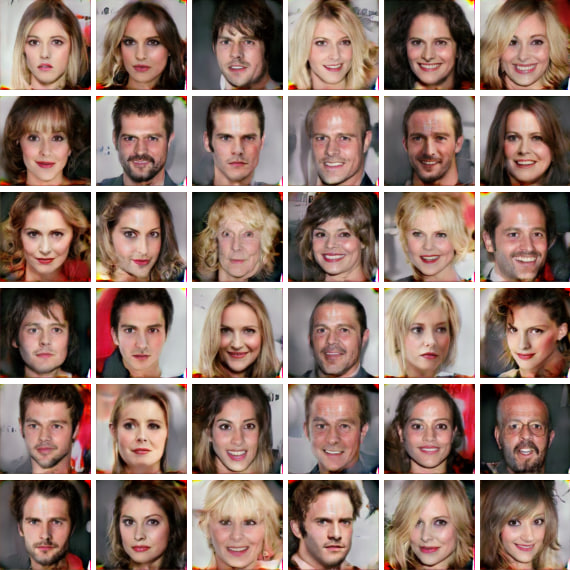}
    \caption*{{\small (b) GeoGAN}}
  \end{minipage}
  \caption{Human faces generated by ConvGAN (\ref{fig: App GANs}a)
    and GeoGAN (\ref{fig: App GANs}b) trained on \CelebAHQ{}
    dataset. Each image is generated as follows. For each of the
    models, we generated 10 images from randomly sampled latent
    points. The image with the highest score from the discriminator is
    added to the canvas. This is repeated 36 times for a
    \(6 \!\times\! 6\) canvas.}
  \label{fig: App GANs}
\end{figure*}
\subsubsection{WGAN-GPs for generating face images}
\label{subsubsec: App WGAN-GP Hands}

\paragraph{Model design} The design of the generator and discriminator
in this section is similar to the generator and discriminator
explained in \cref{Par: GAN-Face Model Design}.

\paragraph{Training detail}
WGAN-GPs use Wasserstein distance for their loss alongside gradient
penalty. Since none of the models showed any improvements after around
100 epochs, we set the number of epochs to 150.

\paragraph{A closer look at generated images} \cref{fig: App WGANs}
portrays a \(6\!\times\!6\) canvas with more images generated by the
WGAN-GPs. Notice the quality, colour, and diversity of the images generated by each of the models.

\subsubsection{WGAN-GPs for generating hand gestures}
\paragraph{Model design}
The design of the generator and discriminator in this section is
similar to the generator and discriminator explained in \cref{Par:
  GAN-Face Model Design}.

\paragraph{Training detail}
Training details are similar to \cref{subsubsec: App WGAN-GP Hands},
except that we run the experiments for 1,000 epochs to make sure all
the models reach peak performance.

\paragraph{A closer look at the generated images} \cref{fig: App
  WGAN-GP Hands} shows the hand gestures generated by both ConvWGAN-GP
and GeoWGAN-GP for each label of the ASL language. These are the same
images as in \cref{fig: WGAN-GP Hands}; however, they have been scaled
up for visualising more details and easier comparison between the images generated by each of the models.

\begin{figure*}[t]
  \centering
  \begin{minipage}[b]{0.48\textwidth}
    \centering
    \includegraphics[width=\textwidth]{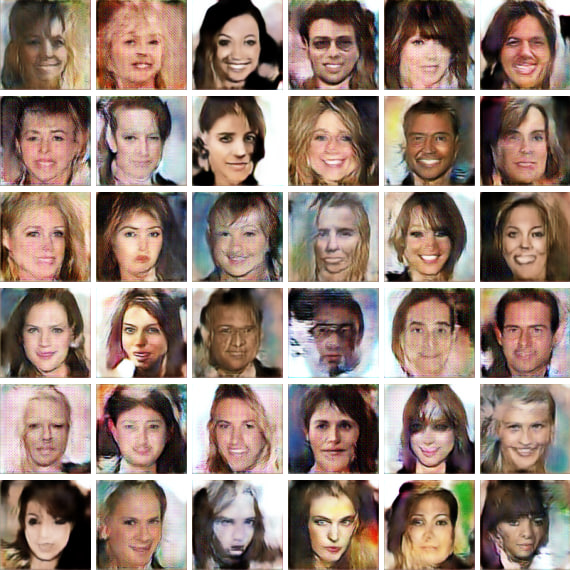}
    \caption*{{\small (a) ConvWGAN-GP}}
  \end{minipage}
  \hfill
  \begin{minipage}[b]{0.48\textwidth}
    \centering
    \includegraphics[width=\textwidth]{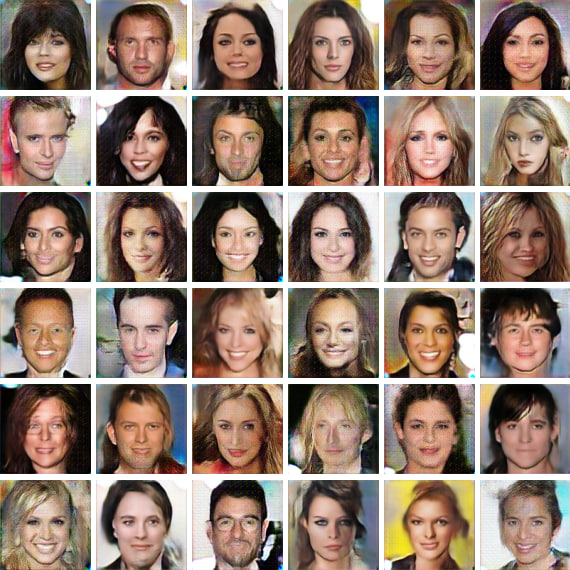}
    \caption*{\small{ (b) GeoWGAN-GP}}
  \end{minipage}
  \caption{{\small Human faces generated by ConvWGAN-GP (\ref{fig: App
        WGANs}a) and GeoWGAN-GP (\ref{fig: App WGANs}b) trained on
      \CelebAHQ{} dataset. Each image is generated as follows. For
      each of the models, we generated 10 images from randomly sampled
      latent points. The image with the highest score from the
      discriminator is added to the canvas. This is repeated 36 times
      for a \(6 \!\times\! 6\) canvas.}}
  \label{fig: App WGANs}
\end{figure*}
\begin{figure*}[h]
  \centering
  \includegraphics[width=\textwidth]{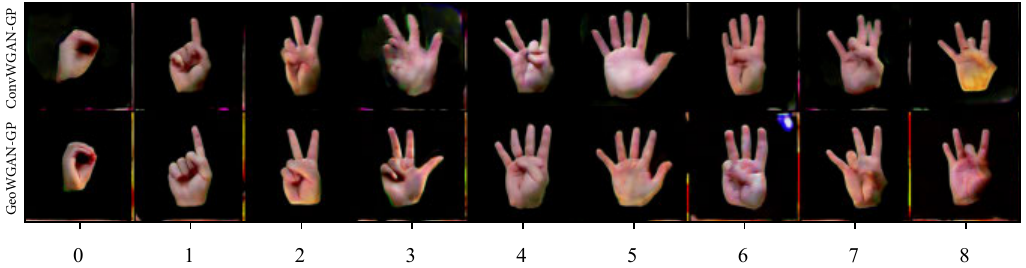}
  \includegraphics[width=\textwidth]{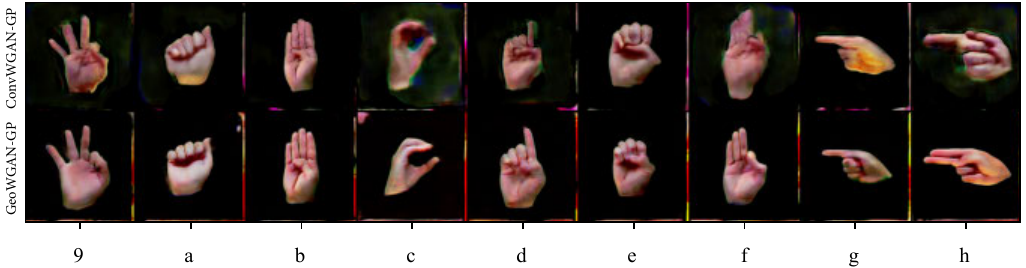}
  \includegraphics[width=\textwidth]{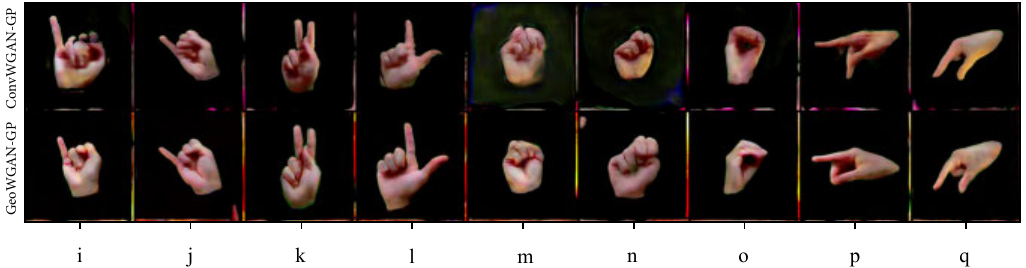}
  \includegraphics[width=\textwidth]{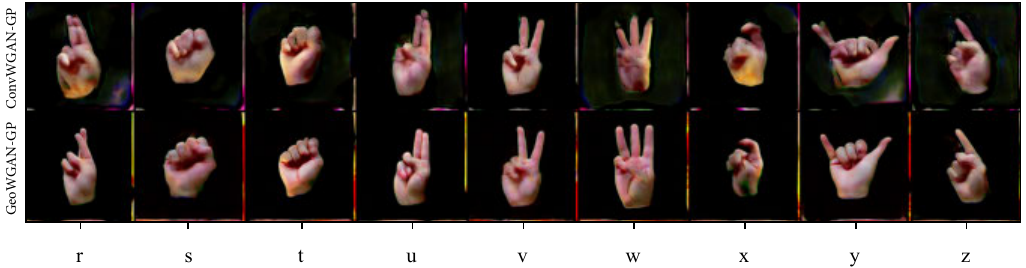}

  \caption{Hand gestures generated by ConvWGAN-GP (first rows), and
    GeoWGAN-GP (second rows), trained on the ASL Hand Gesture
    dataset. These are copies of images included in \cref{fig: WGAN-GP
      Hands} of the main body, scaled up for better comparison.}
  \label{fig: App WGAN-GP Hands}
\end{figure*}
\subsection{VAE}
In this section of the appendix, we discuss the details of the
experiments in \cref{subsec: VAE}.

\subsubsection{Loss function}
\label{app: VAE}S In VAEs, since the quality of generated images is
closely associate to the loss function, we chose a loss function that
helps training a model that not only generates images from the same
distribution as the train images, but also helps generating images
that are sharper and have similar structural similarity. Therefore, we
chose the loss the function to be a combination of
\begin{itemize}
\item \textbf{Binary Cross Entropy (BCE).} BCE loss is used as a
  pixel-wise reconstruction loss in VAEs. It encourages the VAE to
  produce reconstructions that are statistically similar to the input
  data in a pixel-wise manner.
  
\item \textbf{Mean Squared Error (MSE).} MSE penalises large
  pixel-wise differences more heavily and is more sensitive to
  outliers than BCE.

\item \textbf{Mean Absolute Error (MAE).} MAE is less sensitive to
  outliers than MSE. Like MSE, it helps reduce pixel-wise differences
  between input and reconstruction, though the magnitude of errors is
  emphasised differently.

\item \textbf{Multi-scale Structural Similarity (SSIM):} SSIM
  \cite{wang+03} assesses structural similarity between images,
  considering luminance, contrast, and structure. It helps capture
  high-level features and generate images that are structurally more
  similar to the training images.

\item \textbf{Absolute difference of Sobel edge maps:} Sobel edge maps
  highlight edges and gradients in images. Penalising the absolute
  difference between these maps encourages the VAE to reproduce edges
  accurately. It helps improve the sharpness and structural details in
  generated images.
\end{itemize}

\subsubsection{Dataset details}

\paragraph{\CelebA{} dataset}
\CelebA{} dataset is one of the most commonly used datasets in both
generative and discriminative applications in computer vision. This
dataset includes 200k human face images. Each image comes with 40
binary attribute annotations about different features such as
eyebrows, cheeks, nose, hair, eyeglasses, neckties, etc.

\paragraph{ASL Hand Gesture dataset}
Please see \cref{par: app ASL dataset}.

\subsubsection{Conditional VAEs for generating face images}
\label{subsubsec: CVAE-Face}
  
\paragraph{Model design}
Both encoder and decoder are designed according to standard
practices. The encoder first feeds the input image through three
convolution/GeoConv/CoordConv layers consecutively. All these layers
have a kernel size of 3 and a stride of 2 and use ReLU activation. The
result is then flattened and concatenated with the label. Then, we use
two dense layers to learn the mean and standard deviation of the
latent space. Then, a latent is sampled using the normal distribution
with this mean and standard deviation.

This latent and the label are then fed into the decoder which will
generate an image reconstructing the original image. After that, 5
transposed convolution/GeoConv/CoordConv layers consecutively expand
the feature map. Each of those layers has a kernel size of 3 and a
stride of 2 and uses ReLU activation. Finally, a
convolution/GeoConv/CoordConv layer with 3 channels, kernel size of 3
and a stride of 1 with sigmoid activation, synthesises the final
image.

\paragraph{Training detail}
We explained the loss function we use for training the VAEs in
\cref{app: VAE}. During the training, the loss curves start to flatten
out after 20 epochs. Nonetheless, we continued training the VAEs until 30
epochs.

\subsubsection{Conditional VAE for generating hand gestures}
The findings from the experiment on VAEs presented in the main body
show the significant enhancements achieved by incorporating GeoConv
into a VAE. These enhancements are observed both in qualitative and
quantitative performance, as well as in a heightened capacity to
capture the dataset's diversity. In alignment with our experiments in
the GAN section, in this section, we use VAEs for generating images of
ASL hand gestures.

\begin{figure*}[t]
  \centering
  \begin{minipage}[b]{0.33\textwidth}
    \centering
    \includegraphics[width=\textwidth]{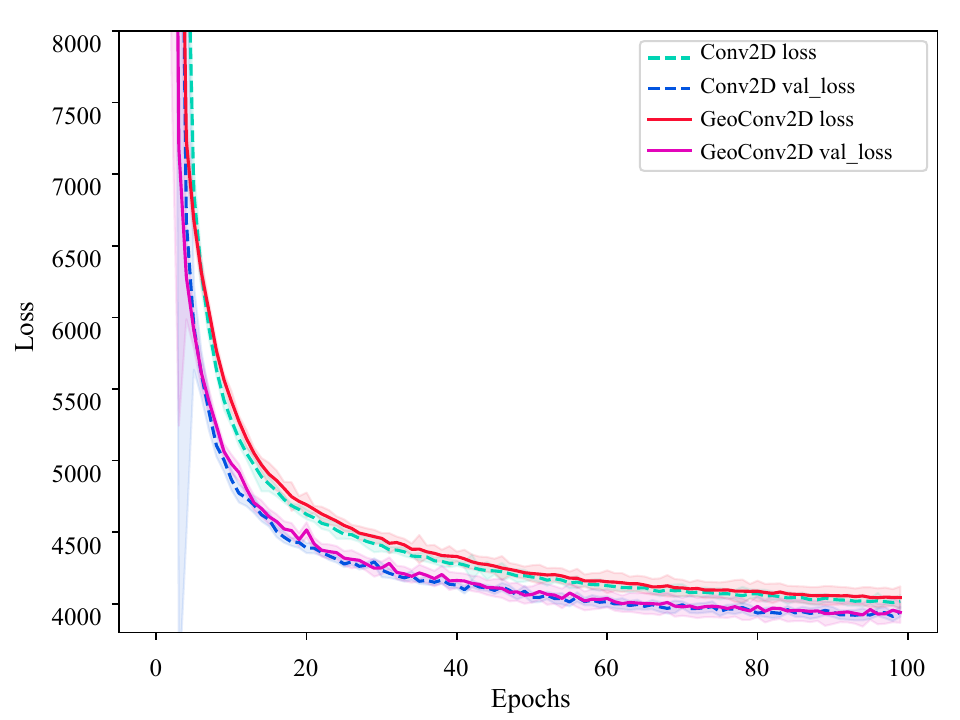}
    \caption*{{\small \(d = 64\)}}
  \end{minipage}
  \hfill
  \begin{minipage}[b]{0.33\textwidth}
    \centering
    \includegraphics[width=\textwidth]{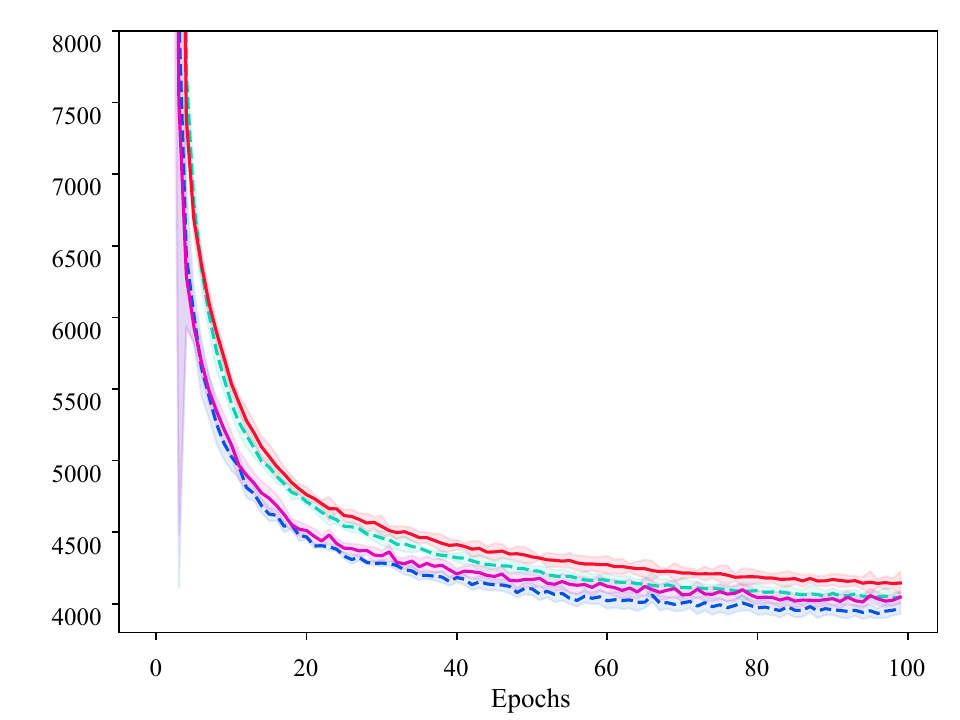}
    \caption*{{\small \(d = 128\)}}
  \end{minipage}
  \hfill
  \begin{minipage}[b]{0.33\textwidth}
    \centering
    \includegraphics[width=\textwidth]{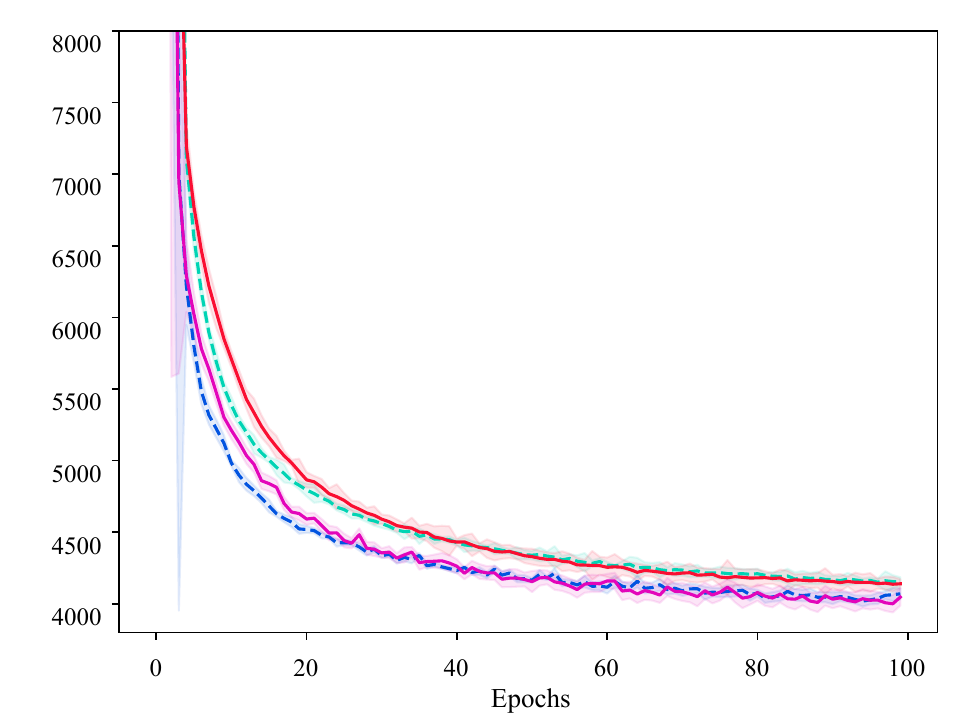}
    \caption*{{\small \(d = 192\)}}
  \end{minipage}
  \caption{{\small Mean and 95\% CI of train and validation losses of GeoVAE
    (red lines), and ConvVAE (dashed blue lines), trained on Hand
    Gesture dataset for latent dimensions \(d \in \{64, 128, 192\}\)
    over five runs with seeds \(0, \dots, 4\) during 100 training
    epochs.}}
  \label{fig: CVAE Evaluation on Hands}
\end{figure*}
\begin{figure*}[h]
  \centering
  \begin{tikzpicture}
    \node[inner sep=0pt] (vanilla-image) at (-0.5,2.5)
    {\includegraphics[width=1.0\textwidth]{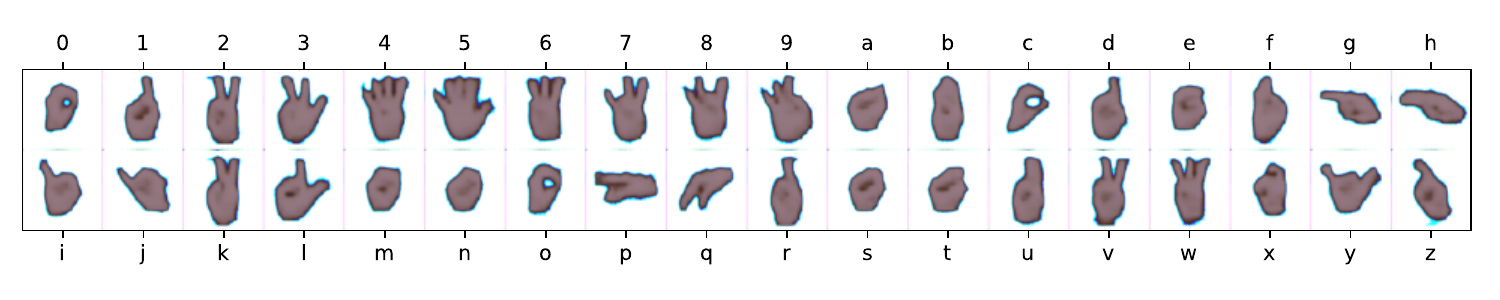}};
    \node[align=center, rotate=90, font=\footnotesize] at ([xshift=-3pt]vanilla-image.west) {ConvVAE};
    
    \node[inner sep=0pt] (geo-image) at (-0.5,-.8)
     {\includegraphics[width=1.0\textwidth]{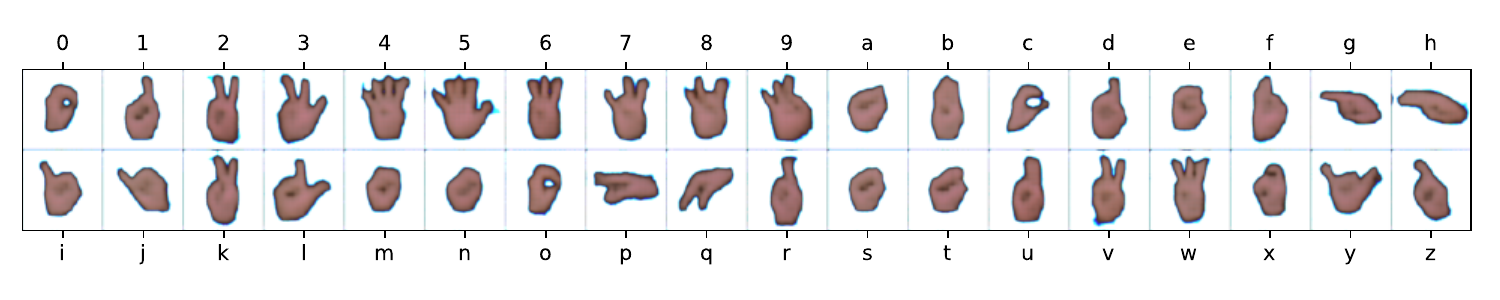}};
    \node[align=center, rotate=90, font=\footnotesize] at ([xshift=-3pt]geo-image.west) {GeoVAE};
  \end{tikzpicture}
  \caption{Hand gestures generated by ConvVAE (top row) and GeoVAE
    (bottom row) with 192-dimensional latent spaces. Images generated
    by GeoVAE have more realistic colours and are slightly sharper.}
  \label{fig: CVAE generations on hands}
\end{figure*}

While \CelebA{} is a vast and diverse collection, comprising
approximately 200k human face images, the hand gesture dataset only
contains just over 2,500 images, each sharing a similar appearance,
primarily differing based on the represented alphabet or
number. Consequently, this dataset introduces a distinct set of
challenges for the VAEs. We train two conditional VAEs on the the
gesture dataset for 100 epochs. We use the same architecture for the
VAEs as in \cref{subsubsec: CVAE-Face}, only differing in some
hyperparameters.

    


We run experiments using latent dimensions 64, 128, and
192. Additionally, each VAE is trained five times, with seeds
\(0, 1,\dots, 4\). The training and validation loss during 100 epochs
of training are visualised in \cref{fig: CVAE Evaluation on Hands}. As
anticipated, training and validation losses are similar for both
architectures across various latent dimensions. As we discussed
before, this is because of the Hand Gesture dataset's small size and
limited diversity.

%

%

\cref{fig: CVAE generations on hands} presents the generated images
produced by each of the conditional VAEs. Both models perform
reasonably well in representing the correct gestures even though they
do not produce high-resolution images compared to GANs. Digging deeper
into the details, images generated by GeoVAE have more realistic
colours and sharper details such as more distinct fingers in comparison to the ConvVAE.

\section{Additional Experiments on Diffusion Models}

To compare GeoConv with CoordConv and Conv2D algorithms in more complex CNN settings, we trained Denoising Diffusion Probabilistic Models (DDPM) on the Smithsonian Butterflies dataset for 30 epochs. The implementation details including architecture and hyperparameters for our diffusion model use the standard DDPM implementation in Keras website\footnote{https://keras.io/examples/generative/ddpm/}. 

\cref{fig: Diffusion Butterflies} show images generated by models based on each architecture. For generating the images, 16 random noises were sampled and fed to all networks. Then they went through the denoising process and the output is presented in these images.

\begin{figure*}[t]
  \centering
  \begin{minipage}[b]{0.32\textwidth}
    \centering
    \includegraphics[width=\textwidth]{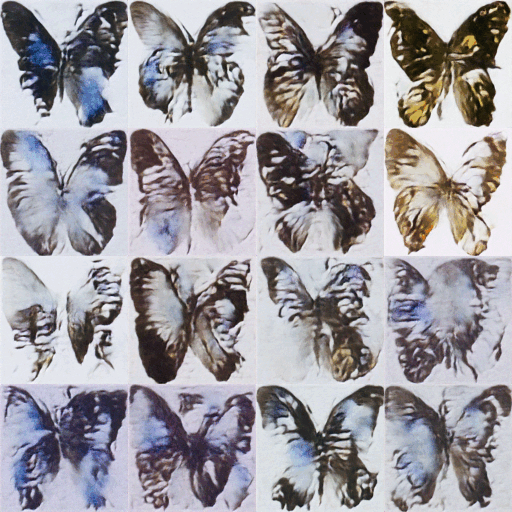}
    \caption*{{\small Conv2D}}
  \end{minipage}
  \hfill
  \begin{minipage}[b]{0.32\textwidth}
    \centering
    \includegraphics[width=\textwidth]{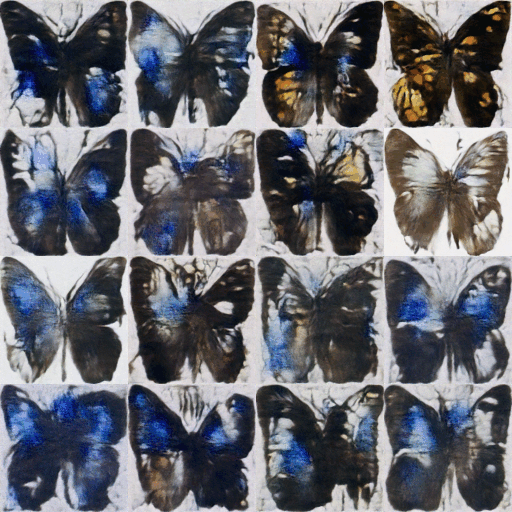}
    \caption*{{\small CoordConv}}
  \end{minipage}
  \hfill
  \begin{minipage}[b]{0.32\textwidth}
    \centering
    \includegraphics[width=\textwidth]{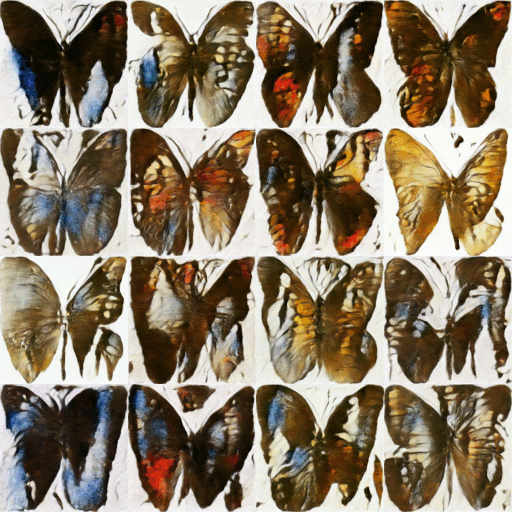}
    \caption*{{\small GeoConv}}
  \end{minipage}
  \caption{{\small Denoising Diffusion Probabilistic Models (DDPM) trained on the Smithsonian Butterflies dataset for 30 epochs. The images generated by GeoConv capture the colour diversity and geometric complexities of butterflies better compared to CoordConv and Conv2D.}}
  \label{fig: Diffusion Butterflies}
\end{figure*}
\begin{table}[ht]
  \centering
  \label{tbl: Quantitative Diffusion Results}
  \begin{tabular}{lccccc}
  Arch.     & Noise loss & Image loss & KID \\
  \midrule
  Conv.     & 0.124 & 0.212 & 0.45 \\
  GeoConv   & 0.097 & 0.146 & 0.35 \\
  CoordConv & 0.110 & 0.161 & 0.41
  \end{tabular}
    \caption{Performance metrics of DDPM models based on studied architectures on the validation set of Smithsonian Butterflies dataset. Mean Absolute Error (MAE) is used for assessing Noise and Image losses. In addition, we also report the Kernel Inception Distance (KID) as a metric for reflecting the quality and diversity of image generation. For all three metrics, lower values mean better results. GeoConv performs favourably to the other two convolutional layers in all metrics.}
\end{table}

\section{Speed analysis}
\subsection{Theoretical analysis: number of FLOPs}
Let us introduce a few notations for each of the variables involved to investigate the number of FLOPs required for performing a forward pass on each Convolutional layer architecture in a two-dimensional space. These include:
\begin{itemize}
  \item Input dimensions: Width ($W$), Height ($H$), and number of input channels ($C_{in}$)
  \item Kernel (or Filter) dimensions: Kernel Width ($K_W$) and Kernel Height ($K_H$)
  \item Parameter specifications: Stride ($S$), Padding ($P$), and number of output channels ($C_{out}$) 
\end{itemize}

Based on this notation, the number of FLOPs in a forward pass for each architecture can be calculated according to the formula in \cref{tab:FLOPs} below.

\begin{table}[h]
  \centering
  \begin{tabular}{ll}
    Architecture & FLOPs\\
    \midrule
    Conv2D       & $2 H_{out} W_{out} K_H K_W C_{in} C_{out}$\\
    GeoConv2D    & $2 H_{out} W_{out} K_H K_W (C_{in} + 1) C_{out}$\\
    CoordConv2D  & $2 H_{out} W_{out} K_H K_W (C_{in} + 2) C_{out}$
  \end{tabular}
  \caption{Number of FLOPs required for each convolutional layer architecture to complete the forward pass. GeoConv adds 50\% less FLOPs compared to CoordConv to the vanilla convolution.}
\label{tab:FLOPs}
\end{table}

where $H_{out}$ and $W_{out}$ are defined as:
\begin{equation}
  \begin{split}
    H_{out} & = [(H - K_H + 2P) / S] + 1\\
    W_{out} & = [(W - K_W + 2P) / S] + 1
  \end{split}
\end{equation}

As these equations indicate, GeoConv adds half as many FLOPs compared to CoordConv to the convolutional layer. It is also worth noting that with higher input dimensions, this superiority even becomes more evident. For example, if the input is 3D, GeoConv adds one-third as many FLOPs compared to CoordConv.

\subsection{Experimental analysis: train and inference}
In addition to the theoretical analysis provided above, here we report the training and inference time for some of our diffusion experiment for further clarification in \cref{tbl: Train and Inference time}.

\begin{table}[ht]
  \centering
  \begin{tabular}{lccccc}
  Arch.     & Train Time & Inference time \\
  \midrule
  Conv.     & 33.7 & 0.145 \\
  GeoConv   & 38.3 & 0.229 \\
  CoordConv & 41.8 & 0.308
  \end{tabular}
    \caption{Train time per one epoch of training and Inference time (50 denoising steps) for generating one image for each model based on different architectures in the Denoising Diffusion Probabilistic Models (DDPM) experiment. As expected vanilla convolution is the fastest but GeoConv has a 8.4\% faster training time and 25.6\% faster inference speed compared to CoordConv.}
  \label{tbl: Train and Inference time}
\end{table}
\section{Proofs}
\label{sec: Proofs}
\begin{proof}[Proof of \cref{thm: Filter Collapse}]
  Let us use the same notation as in \cref{thm: Positional
    Dependency}.  Since the proof is similar for all coordinate
  channels, we only prove this for the first channel. Let
  \(f = (f_{i_1,\dots,i_n})\) be the convolution filter corresponding
  the first coordinate channel \(c\) in CoordConv. Let
  \begin{equation}
    \bar{f}_{i_1} = \sum_{i_2, \dots, i_n} f_{i_1, \dots, i_n},
  \end{equation}
  where \(1 \leq i_k \leq s_k\) for \(1 \leq k \leq n\). At each step
  of the convolution operation, we have that
  \begin{equation}
    \label{eq: Filter Collapse}
    \begin{split}
    \sum_{i_1, \dots, i_n} & f_{i_1, \dots, i_n} c_{i_1 + j_1, \dots, i_n + j_n}\\
    =&\:\ \sum_{i_1} (\sum_{i_2, \dots, i_n} f_{i_1, \dots, i_n}) c_{i_1 + j_1, j_2, \dots, j_n}\\
    =&\:\ \sum_{i_1} \bar{f}_{i_{\ell}} c_{i_1 + j_1, j_2, \dots, j_n}.
    \end{split}
  \end{equation}
  Hence, the \(s_1 \times \dots \times s_n\) filter \(f\) does not
  extract any more information from the first coordinate channel \(c\)
  than the \(s_1 \times 1 \times \dots \times 1\) filter
  \(\bar{f} = (\bar{f}_{i_1})\).
\end{proof}

\begin{proof}[Proof of \cref{thm: Equivalence}]
  Let us use the notation in the proofs of \cref{thm: Positional Dependency,thm: Filter Collapse}. We use \(\vi\) to refer to the tuple \((i_1, \dots, i_n)\) and drop the \(j\) indices (Similar to those appearing in \cref{eq: Filter Collapse}) for the sake of brevity. Now, if we denote the filters by \(f\), input tensor by \(x\), and coordinate channels in CoordConv by \(c\), then, we have that
  \begin{equation}
    \label{eq: CoordConv}
    f * (x, c) = f^{(1, \dots, k)} \! * x \ + \ f^{(k+1, \dots, k+n)} \! * c.
  \end{equation}
  Similarly for GeoConv, if show the GeoPos channel by \(g\), then we have that
  \begin{equation}
    \label{eq: GeoConv}
    f * (x, g) = f^{(1, \dots, k)} \! * x \ + \ \bar{f} \! * g,
  \end{equation}
  where \(\bar{f}\) is in fact \(f^{(k+1)}\); however, to avoid confusion with \(f^{(k+1)}\) in \cref{eq: CoordConv}, we use \(\bar{f}\) for the filter corresponding to the GeoPos channel.

  Now, we need to prove that for any \(f^{(k+1, \dots, k+n)}\) with \(s_1, \dots, s_n \geq 2\) kernel size, there exists \(\bar{f}\) of the same kernel size, such that
  \begin{equation}
    \label{eq: GeoPos vs CoordPos}
    f^{(k+1, \dots, k+n)} \! * c = \bar{f} \! * g.
  \end{equation}

  The LHS of \cref{eq: GeoPos vs CoordPos} can be expanded as
  \begin{equation}
    \label{eq: CoordConv Expansion}
    \begin{split}
      f^{(k+1, \dots, k+n)} \! * c
      & = \sum_{\vi} f_{\vi}^{(k+1, \dots, k+n)} c_{\vi + \vj}^{(1, \dots, n)}\\
      & = \sum_{t=1}^n \sum_{\vi} f_{\vi}^{(k+t)} c_{\vi + \vj}^{(t)}\\
      & = \sum_{t=1}^n \sum_{i_t} \sum_{\vi\backslash i_t} f_{\vi}^{(k+t)} c_{\vi + \vj}^{(t)}\\
      & = \sum_{t=1}^n \sum_{i_t} \left(\sum_{\vi\backslash i_t} f_{\vi}^{(k+t)}\right) c_{\vi + \vj}^{(t)}\\
    \end{split}
  \end{equation}

  The RHS of \cref{eq: GeoPos vs CoordPos} can be expanded as
  \begin{equation}
    \label{eq: GeoConv Expansion}
    \begin{split}
      \bar{f} \! * g
      = \sum_{\vi} \bar{f} g_{\vi} & = \frac{1}{t} \sum_{t=1}^n \sum_{\vi} \bar{f}_{\vi} c_{\vi + \vj}^{(t)}\\
      & = \frac{1}{t} \sum_{t=1}^n \sum_{i_t} \sum_{\vi\backslash i_t} \bar{f}_{\vi} c_{\vi + \vj}^{(t)}\\
      & = \frac{1}{t} \sum_{t=1}^n \sum_{i_t} \left(\sum_{\vi\backslash i_t} \bar{f}_{\vi}\right) c_{\vi + \vj}^{(t)}\\
    \end{split}
  \end{equation}

  Thus for \cref{eq: GeoPos vs CoordPos} to hold, it sufficient that
  \begin{equation}
  \label{eq: Linear}
    \sum_{\vi\backslash i_t} f_{\vi}^{(k+t)} = \frac{1}{t} \sum_{\vi\backslash i_t} \bar{f}_{\vi}, \quad t = 1, \dots, n.
  \end{equation}
   have solution in \(\bar{f}\). \cref{eq: Linear} is a linear equation in \(\bar{f}\) with \(s_1 s_2 \cdots s_n\) variables and \(n (s_1 + s_2 + \cdots + s_n)\) equations. Thus, if \(s_1 s_2 \cdots s_n \geq n (s_1 + s_2 + \cdots + s_n)\), \cref{eq: Linear} is guaranteed to have solutions and GeoConv is equivalent to CoordConv.
\end{proof}
%



\end{document}